\documentclass{article}

\usepackage{graphicx} % more modern
\usepackage{subfigure}

\usepackage{natbib}

\usepackage{algorithm}
\usepackage{algorithmic}

\usepackage{hyperref}
\usepackage{breakurl}

\usepackage{times,amsmath,amsfonts,stmaryrd,amsthm,natbib}
\usepackage{csvsimple}

\usepackage{enumitem}
\newtheorem{theorem}{Theorem}[section]
\newtheorem{lemma}[theorem]{Lemma}

\renewcommand{\eqref}[1]{Eq.~(\ref{#1})}
\newcommand{\figref}[1]{Figure~\ref{#1}}
\newcommand{\tabref}[1]{Table~\ref{#1}}        

\newcommand{\thmref}[1]{Theorem~\ref{#1}}
\newcommand{\lemref}[1]{Lemma~\ref{#1}}

\newcommand{\appref}[1]{Appendix~\ref{#1}}

\newcommand{\algref}[1]{Alg.~\ref{#1}}

\newcommand{\mt}[1]{\mathbb{#1}}
\newcommand{\vc}[1]{\mathbf{#1}}

\renewcommand{\P}{\mathbb{P}}
\newcommand{\E}{\mathbb{E}}
\newcommand{\reals}{\mathbb{R}}
\newcommand{\nats}{\mathbb{N}}

\newcommand{\Diag}{\mathrm{Diag}}

\newcommand{\Var}{\mathrm{Var}}

\newcommand{\cA}{\mathcal{A}}
\newcommand{\cD}{\mathcal{D}}

\newcommand{\cR}{\mathcal{R}}

\newcommand{\br}{\mathbf{r}}

\newcommand{\ignore}[1]{}
\DeclareMathOperator*{\argmin}{argmin}
\DeclareMathOperator*{\argmax}{argmax}
\newcommand{\loss}{\ell}

\newcommand{\one}{\mathbb{I}}
\newcommand{\norm}[1]{\|#1\|}
\newcommand{\dotprod}[1]{\langle #1 \rangle}

\newcommand{\cov}{b}
\newcommand{\brtrain}{{\mathbf{k}}}
\newcommand{\rtrain}{k}

\newcommand{\obj}{\mathrm{obj}}

\newcommand{\submat}{\mathrm{sub}}

\usepackage[accepted]{icml2013.arxiv}

\newcommand{\ourtitle}{Feature Multi-Selection among Subjective Features}

\newcommand{\arxiv}[1]{#1}
\newcommand{\ICML}[1]{}

\begin{document}

\twocolumn[
\icmltitle{\ourtitle}

% unless you've provided the [accepted] option to the icml2013
\icmlauthor{Sivan Sabato}{sivan.sabato@microsoft.com}
\icmlauthor{Adam Kalai}{adam.kalai@microsoft.com}
\icmladdress{Microsoft Research New England, 1 Memorial Drive, Cambridge, MA, USA}

\icmlkeywords{}

\vskip 0.3in
]

\begin{abstract}
When dealing with subjective, noisy, or 
otherwise nebulous features, the ``wisdom of crowds'' suggests that one may benefit from multiple judgments of the 
same feature on the same object.  
We give theoretically-motivated {\em feature multi-selection} algorithms that choose, among a large set of candidate features,
not only which features to judge but how many times to judge each one. We demonstrate the effectiveness of this approach for linear 
regression on a crowdsourced learning task of predicting people's height and weight from photos, using features such as {\em gender} and
{\em estimated weight} as well as culturally fraught ones such as {\em attractive}.
This work has been published in \citet{SabatoKa13b}.
\end{abstract}

\section{Introduction}

In this paper we consider prediction with subjective, vague, or noisy attributes (which are also termed `features' throughout this paper). Such attributes can sometimes be useful for prediction, because they account for an important part of the signal that cannot be otherwise captured. In a crowdsourcing setting, the ``wisdom of crowds'' suggests that including multiple assessments of the same feature by different people may be useful. Henceforth, we refer to assessments of features as {\em judgments}. This paper introduces the problem of selecting, from a set of candidate features, which ones to use for prediction, and \emph{how many judgments to acquire for each}, for a given budget limiting the total number of judgments.  We give theoretically justified algorithms for this problem, and report crowdsourced experimental results, in which judgments of highly subjective features (even culturally fraught ones such as {\em attractive}) are helpful for prediction.

As a toy example, consider the problem of estimating the number of jelly beans in a jar based on an image of the jar.  A linear regressor with multiple judgments of features might have the form,
\begin{align*}
\hat{y} =& 0.95(\text{est. number of beans})^{/5} - 50(\text{round jar})^{/2}+\\
&100(\text{monochromatic})^{/1}+30(\text{beautiful})^{/3}.\end{align*}
Here, for binary attributes, $a^{/r_a} \in [0,1]$ denotes the fraction of positive judgments out of $r_a$ judgments of attribute $a$. For real-valued attributes, $a^{/r_a}$ denotes the mean of $r_a$ judgments. The shape, number of colors, and attractiveness of the jar each help correct biases in the estimated number of beans, averaged across five people. Our goal is to choose a regressor that, as accurately as possible, estimates the labels (i.e., jelly bean counts) on future objects (i.e., jars) drawn from the same distribution, while staying within a budget of feature judgment resources per evaluated object at test time.
In the example above, notice that even though the  {\em monochromatic} coefficient is greater than the {\em beautiful} coefficient, fewer monochromatic judgments are used, because counting the number of colors is more objective, and hence further judgments are less valuable.  While this example is contrived, similar phenomena are observed in the output of our algorithms (see \ref{tab:pred}).

We refer to the problem of selecting the number of repetitions, $r_a$, of each attribute, as the {\em feature multi-selection} problem, because it generalizes the feature selection problem of choosing a subset of features, i.e., $r_a \in \{0,1\}$, to choosing a multiset of features, i.e., $r_a \in \mathbb{N}$.
Since the feature selection problem is well known to be NP-hard \cite{Natarajan95}, our problem is also NP-hard in the general case.  (For a formal reduction, one simply considers the ``objective'' case where all judgments of the same feature-object pair are identical.)  Nonetheless, several successful approaches have been proposed for feature selection. The algorithms that we propose generalize two of these approaches to the problem of feature multi-selection.

Our algorithms are theoretically motivated, and tested on synthetic and real-world data.  The real world data are photos extracted from the publicly available \href{http://www.cockeyed.com/photos/bodies/heightweight.html}{Photographic Height/Weight Chart}\footnote{\href{http://www.cockeyed.com/photos/bodies/heightweight.html}{http://www.cockeyed.com/photos/bodies/heightweight.html}}, where people post pictures of themselves announcing their own height and weight.

As a more general motivation, consider a scientist who would like to use crowdsourcing as an alternative to themselves estimating a value for each of a large data set of objects. Say the scientist gathers multiple judgments of a number of binary or real-valued attributes for each object, and uses linear regression to predict the value of interest. In some cases, crowdsourcing is a natural source of judgments, as a great number of them may be acquired on demand, rapidly, and at very low cost.  We assume the scientist has access to the following information:
\begin{itemize}
\item A {\bf labeled set} of objects $(o,y) \in O\times \mathcal{Y}$ (with no judgments), where $O$ is a set of {\em objects} and $\mathcal{Y}\subseteq \reals$ is a set of ground-truth {\em labels} drawn independently from a distribution $\mathcal{D}$.
\item A {\bf crowd}, which is a large pool of {\em workers}.
\item A possibly large set of candidate {\bf attributes} $\mathcal{A}$. For any attribute $a \in \mathcal{A}$ and object $o \in O$, the judgment of a random worker from the crowd may be queried at a cost.
\item A {\bf budget} $B$, limiting the number of attribute judgments to be used when evaluating the regressor on a new unseen object.
\end{itemize}

Our approach is as follows:
\begin{enumerate}
\item Collect $k \geq 2$ judgments for each candidate attribute in $\cA$, for each object in the labeled set.
\item Based on this data and the budget, decide {\em how many judgments} of each attribute to use in the regressor.\label{step:multiselect}
\item Collect additional judgments (as needed) on the labeled set so that each attribute has the number of judgments specified in the previous step. \label{step:collectjudge}
\item Find a linear predictor based on the average judgment of each feature.\footnote{We focus on mean averaging, leaving to future work other aggregation statistics such as the median.} \label{step:predict}
\end{enumerate}
Step 4 can be accomplished by simple least-squares regression. The goal in Step 2 (feature multi-selection) is to decide on a number of judgments per attribute that will hopefully yield the smallest squared error after Step 4.  

Interestingly, even given as few as $k=2$ judgments per attribute, one can project an estimate of the squared error with more than $k$ judgments of some features.  We prove that these projections are accurate, for any fixed $k \geq 2$, as the number of labeled objects increases.  Our algorithms perform a greedy strategy for feature multi-selection, to attempt to minimize the projected loss. This greedy strategy can be seen as a generalization of the Forward Regression approach for standard feature selection \citep[see e.g.][]{Miller02}.
The first algorithm operates under the assumption that different attributes are uncorrelated. In this case the projection simplifies to a simple scoring rule, which incorporates attribute-label correlations as well as a natural notion of inter-rater reliability for each attribute.  In this case, greedy selection is also provably optimal. While attributes are highly correlated in practice, the algorithm performs well in our experiments, possibly because Step 4 corrects for a small number of poor choices during feature multi-selection. The second algorithm attempts to optimize the projection without any assumptions on the nature of correlations between features. 

While crowdsourcing is one motivation, the algorithms would be applicable to other settings such as learning from noisy sensor inputs, where one may place multiple sensors measuring each quantity, or social science experiments, where one may have multiple research assistants (rather than a crowd) judging each attribute.

The main contributions of this paper are: (a) introducing the feature multi-selection problem, (b) giving theoretically justified feature multi-selection algorithms, and (c) presenting experimental results, showing that feature multi-selection can yield more accurate regressors, with different numbers of judgments for different attributes. \ICML{Proofs of results and additional experimental data are provided in \citet{SabatoKa13}.}

\section*{Related Work}

Related work spans a number of fields, including Statistics, Machine Learning, Crowdsourcing, and measurement in the social sciences.  A number of researchers have studied {\em attribute-efficient prediction} (also called {\em budgeted learning}) assuming, as we do, that there is a cost to evaluating attributes and one would like to evaluate as few as possible (see, for instance, the recent work by \citet{CSS11} and references therein).  In that line of work, each attribute is judged at most once. The errors-in-variables approach \citep[e.g.,][]{ChengVa99} in statistics estimates the `true' regression coefficients using noisy feature measurements. This approach is less suitable in our setting, since our final goal is to predict from noisy measurements.

A wide variety of techniques have been studied to combine estimates of experts or the crowd of a single quantity of interest \citep[see, e.g.][]{DawidSkene79, SFB+94,MultidimensionalWisdom}, like estimating the number of jelly beans in a jar from a number of guesses.

Two recent works on crowdsourcing are very relevant.  \citet{PH12} crowdsourced the mean of 3 judgments of each of 102 binary attributes on over 14,000 images, yielding over 4 million judgments.  Some of their attributes are subjective, e.g., {\em soothing}. We employ their crowdsourcing protocol to label our binary attributes.  \citet{Memorability} study subjective and objective features for the task of estimating how memorable an image is, by taking the mean of 10 judgments per attribute for each image.  They perform greedy feature selection over these attributes to find the best compact set of attributes for predicting memorability.  The key difference between their algorithm and ours is that theirs does not choose how many judgments to average.  Since that quantity is fixed for each attribute, their setting falls under the more standard feature selection umbrella.  In our experiments we compare this approach to our algorithms. 

Finally, in the social sciences, a wide array of techniques have been developed for assessing inter-rater reliability of attributes, with the most popular perhaps being the $\alpha$ coefficient \cite{Alpha}.  A principal use of such measures is determining, by some threshold, which features may be used in content analysis.  For an overview of reliability theory, see \cite{ContentAnalysis}.

\section{Preliminary assumptions and definitions}

Let there be $d$ candidate attributes called $\mathcal{A} = [d]=\{1,2,\ldots,d\}$.  We assume that, for any object $o$ and attribute $a$, there is a distribution over judgments $\P[X[a] \mid O=o]$, and we assume that the judgments of attribute-object pairs are conditionally independent given the sets of attributes and objects. This represents an idealized setting in which a new random crowd worker is selected for each attribute-object judgment (In our experiments, we limit the total amount of work that any one worker may perform).
We assume a distribution $\cD$ over labeled objects, where labels are real numbers. We denote by $\cD_O$ the marginal distribution over objects drawn according to $\cD$. We let $\P[X[a]] = \P_{O\sim \cD_O}[X[a] \mid O]]$.
Labels $y$ are assumed to be real valued.  As is standard, we assume one ``true'' label $y_i$ for each object $o_i$.

For notational ease, we assume that in the feature multi-selection phase, exactly $k \geq 2$ judgments for each feature are collected.  Our analysis trivially generalizes to the setting in which different attributes are judged different numbers of times.  Finally, each attribute $a$ is assumed to have an expected value of $\E[X[a]]=0$, where the expectation is taken across objects and judgments of $a$.  This is done for ease of presentation, so that we do not have to track the mean vectors as well as the variance. When discussing implementation details, we describe how to remove this assumption in practice without loss of generality.

Vectors will be boldface, e.g., $\vc{x} = (x[1],\ldots,x[d])$,
random variables will be capitalized, e.g., $X$, and
matrices will be in black-board font, e.g., $\mt{X}$. The $i$'th standard unit vector is denoted by $\vc{e}_i$.

Let $\br \in \nats^d$ represent the number of judgments for each feature, so that attribute $a$ is judged $r[a]$ times,
and we represent the object's judgments by $\vc{x}$, defined as:
\[\vc{x}=\left( \langle x[1](j) \rangle_{j=1}^{r[1]}, \ldots, \langle x[d](j) \rangle_{j=1}^{r[d]}\right),
\]
where $x[a](j)$ is the $j$th judgment of attribute $a$ in $\vc{x}$, and 
$\langle x[a](j) \rangle_{j=1}^{r[a]}$ is a vector with $x[a](j)$ in coordinate $j$.  We say that $\br$ is the \emph{repeat vector} of $\vc{x}$. We denote the set of all possible representations with repeat vector $\br$ by $\reals^{[\br]}$.

We denote by $D_\br$ the distribution which draws $(\vc{X},Y) \in \reals^{[\br]} \times \reals$ by first drawing a labeled object $(O,Y)$ from $\cD$, and then drawing a random representation $\vc{X}\in \reals^{[\br]}$ for this object.
We denote by $D_\infty$ the distribution that draws $(\vc{X},Y)$ where $\vc{X}\in \reals^d$ by first drawing $(O,Y)$ from $\cD$ and then setting $X[a] = \E[X[a] \mid O]$. We denote the expectation over $D_\br$ by $\E_\br = \E_{(\vc{X},Y)\sim D_\br}$. For $D_\infty$ we denote $\E_\infty = \E_{(\vc{X},Y)\sim D_\infty}$.

For $k \geq 2$, let $\brtrain =(k,k,\ldots,k) \in \nats^d$ be the repeat vector used in the first training phase. The feature multi-selection algorithm receives as input a labeled training set $S = ((\vc{x}_1,y_1),\ldots,(\vc{x}_m,y_m))$ where $\vc{x}_i \in \reals^{kd}$
and $y_i \in \reals$, drawn from $D_\brtrain$. This sample is generated by first drawing a set of labeled objects $((o_1,y_1),\ldots,(o_m,y_m))$ i.i.d. from $\cD$, and then drawing a random representation $\vc{x}_i$ for object $o_i$.
The algorithm further receives as input a budget $B \in \nats$, which specifies the total number of feature judgments allowed for each unlabeled object at test (i.e., prediction) time.
The output of the algorithm is a new vector of repeats $\br \in R_B$, where,
\[
R_B \equiv \left\{ \br \in \nats^d \mid {\sum}_{a \in \mathcal{A}} r[a] \leq B\right\}.
\]

Let $o$ be an object with a true label $y$, and let $\hat{y}$ be a prediction of the label of $o$.
The squared loss for this prediction is $\loss(y,\hat{y}) = (y - \hat{y})^2$. Given a function $f:Z \rightarrow \reals$
for some domain $Z$, and a distribution $D$ over $Z \times \reals$, we denote the average loss of $f$ on $D$ by
\[
\loss(f, D) \equiv \E_{(Z,Y)\sim D}[\loss(f(Z),Y)]. % Sivan, I changed this because I think this is what you mean but correct me if I'm wrong.
\]
The final goal of our procedure is to find a predictor with a low expected loss on labeled objects drawn from $\cD$.
This predictor must use only $B$ feature judgments for each object, as determined by the test repeat vector $\br$.
We consider linear predictors $\vc{w} \in \reals^d$ that operate on the vector of \emph{average} judgments of $\vc{x}\in \reals^{[\br]}$, defined as follows:
\[
\bar{x}[a] \equiv \begin{cases}\frac{1}{r[a]}\sum_{j=1}^{r[a]} x[a](j) & \text{if }r[a]>0,\\
0 &\text{if }r[a]=0.\end{cases}
\]
For an input representation $\vc{x}$, the predictor $\vc{w}$ predicts the label $\dotprod{\vc{w},\bar{\vc{x}}}$.
For vector $\vc{v} \in \reals^d$, we denote by $\Diag(\vc{v})\in \reals^{d \times d}$ the diagonal matrix with $v[a]$ in the $a$th position. 

For a vector $\br \in \nats^d$ and a matrix $\mt{S}\in \reals^{d\times d}$, we denote by $\submat_\br(\mt{S})$ the submatrix of $\mt{S}$ resulting from deleting all rows and columns $a$ such that $r[a] = 0$. For a vector, $\submat_\br(\vc{u})$ omits entries $a$ such that $r[a] = 0$.  Here $\submat_\br(\vc{u})\in \reals^{d'}$ and $\submat_\br(\mt{S})\in\reals^{d'\times d'}$, where $d'$ is the support size of $\br$.
We denote the pseudo-inverse of a matrix $\mt{A} \in \reals^{n\times n}$ \citep[see e.g.][]{BenIsraelGr03} by $\mt{A}^+$.

\section{Feature Multi-Selection Algorithms}

The input to a feature multi-selection algorithm is a budget $B$ and $m$ labeled examples in which each attribute has been judged $k$ times, and the output is a repeat vector $\br \in R_B$.  Our ultimate goal is to find $\br$ and a predictor $\vc{w} \in \reals^d$ such that $\loss(\vc{w},D_\br)$ is minimal. We now give intuition about the derivation of the algorithms, but their formal definition is given in \algref{alg:all}.

Define the loss of a repeat vector to be $\loss(\br) \equiv \min_{\vc{w} \in \reals^d} \loss(\vc{w}, D_\br)$. The goal is to minimize $\loss(\br)$ over $\br \in R_B$.
We give two forward-selection algorithms, both of which begin with $\br = (0,\ldots,0)$ and greedily increment $r[a]$ for $a$ that most decreases an estimate of $\loss(\br)$.  The key question is how does one estimate this projected loss $\loss(\br)$ since the number of judgments can exceed $k$.
We simplify notation by first considering only $\br$ which are {\em positive}, i.e., $r[a]\geq 1$ for each $a$.  We will shortly explain how to handle $r[a]=0$.  Define 
\begin{align*}
\vc{\cov} = \E[\vc{X}Y]\text{, and }\Sigma_\br = \E_\br[\bar{\vc{X}}^T\bar{\vc{X}}].
\end{align*}
We call $\cov[a]$ the {\em correlation} of $a$ with the label. Note that $\vc{\cov} = \E_{\brtrain}[\bar{\vc{X}}Y]$, since linearity of expectation implies that $\vc{\cov}$ does not depend on $\brtrain$.
Straightforward calculations show that, for any positive repeat vector $\br$, If $\Sigma_\br$ is non-singular,\footnote{For singular $\Sigma_\br$, the pseudo-inverse $\Sigma_\br^+$ replaces $\Sigma_\br^{-1}$.} 
\[
\loss(\br)=\min_\vc{w} \E_\br\left[(\vc{w}^T \vc{\bar{X}}-Y)^2\right] = \E_\br[Y^2]-\vc{\cov}^T \Sigma_\br^{-1}\vc{\cov}.
\]
Since $\E[Y^2]$ does not depend on $\br$, minimizing $\loss(\br)$ is equivalent to maximizing $\vc{\cov}^T \Sigma_\br^{-1} \vc{\cov}$ (for positive $\br$ and nonsingular $\Sigma_\br$).

\subsection{A Scoring Algorithm}
The first algorithm that we propose is derived from the zero-correlation assumption, that $\E[X[a]X[a']]=0$ for $a\neq a'$, or equivalently that the covariance matrix is diagonal.  
 Perhaps the simplest approach to standard feature selection is to score each feature independently, based on its normalized empirical correlation with the label, and to select the $B$ top-scoring features. If features are uncorrelated and the training sample is sufficiently large, then this efficient approach finds an optimal set of features. The feature multi-selection scoring algorithm that we propose henceforth is optimal under similar assumptions, however it is complicated by the fact that we may include multiple repetitions of each feature.
Under the zero-correlation assumption, $\Sigma_\br$ is diagonal, and its $a$th element, for $r[a] > 0$, can be expanded as
\begin{align}
\E_\br[(\bar{X}[a])^2] &= \sigma^2[a] + \frac{v[a]}{r[a]}, \text{ where}\notag\\
v[a] &\equiv \E_{O \sim \cD_O}[\Var[X[a] \mid O]] \text{ and} \nonumber\\
\sigma^2[a] &\equiv \E_\infty\left[(X[a])^2\right]. \nonumber
\end{align}
We refer to $v[a]$ as the {\em internal variance} as it measures the ``inter-rater reliability'' of $a$, and we call $\sigma^2[a]$ the {\em external variance} as it is the inherent variance between examples.  Hence for a diagonal $\Sigma_\br$, simple manipulation gives,
\begin{equation}\label{eq:score}
\E[Y^2]-\loss(\br) = \sum_{a: r[a] > 0} \frac{(\cov[a])^2}{\sigma^2[a]+\frac{v[a]}{r[a]}}.
\end{equation}
Therefore, when $\Sigma_\br$ is diagonal,  minimizing the projected loss is equivalent to maximizing the RHS above, a sum of independent terms that depend on the correlation and on the internal and external variance of each attribute, all of which can be estimated just once, for all possible repeat vectors.  As one expects, greater correlation indicates a better feature, while a greater external variance indicates a worse feature. A larger internal variance indicates that more repeats are needed to achieve prediction quality.

To estimate \eqref{eq:score} we estimate each of the components on the RHS. Unbiased estimation of $\vc{b}$ is straightforward, and unbiased estimation of $\vc{v}$ is also possible for $k\geq 2$ samples per object, though importantly one should use the unbiased variance estimator,
\begin{align}\label{eq:hatv}
&\hat{v}[a] = \frac{1}{m}\sum_{i} \mathrm{VarEst}(x_i[a](1),\ldots,x_i[a](j))\text{,}\\
&\mathrm{VarEst}(\alpha_1,\ldots,\alpha_n) \equiv \frac{1}{n-1}\sum_{j\in[n]} (\alpha_j - \frac{1}{n}\sum_{j' \in [n]}\alpha_{j'})^2.\notag
\end{align}

Using these estimates of $\vc{v}$, we estimate the external variance using the equality  $\sigma^2[a]=\E_{\bf k}\left[(\bar{X}[a])^2\right]-\frac{v[a]}{k}$.  A slight complication arises here, as this estimate might be negative for small samples, so we round it up to 0 when this happens. Another issue might seem to arise when the denominator of one of the summands in \eqref{eq:score} is zero, however note that this can only occur if both the internal and the external variance are zero, which implies that the feature is constantly zero, thus zeroing its correlation as well. The same holds for the estimated ratio. In such cases we treat the ratio as equal to $0$.

\subsection{The Full Multi-Selection Algorithm}
The scoring algorithm is motivated by the assumption of zero correlation between features. However, this assumption rarely holds in practice.
Building on and paralleling the definitions and derivation above, the {\em Full Algorithm} similarly maximizes $\vc{\cov}^T \Sigma_\br^{-1} \vc{\cov}$ without this assumption.  For positive $\br$, one has
\begin{align*}
\Sigma_{\br} &= \Sigma + \Diag(v[1]/r[1],\ldots, v[d]/r[d])
\end{align*}
Where $\Sigma \equiv \E_\infty[\vc{X}^T\vc{X}]$ is the {\em external covariance} matrix, and we estimate it based on the equality $\Sigma = \Sigma_\brtrain - \Diag(\vc{v})/\rtrain$. Just as in the Scoring algorithm, the estimates of $\sigma^2[a]$ might be negative, in the full algorithm it is possible that the estimate of $\Sigma$ will not be positive semi-definite, so we analogously ``round up'' our estimate of $\Sigma$ to the nearest PSD matrix (see implementation details below). The estimate when some of the $r[a]$'s are zero is formed by deleting the corresponding entries in the estimate of $\vc{\cov}$ and the corresponding rows and columns in the estimate of $\Sigma_\br$.

\begin{algorithm}[t]
 \begin{algorithmic}[1]
  \caption{Feature multi-selection algorithms} \label{alg:all}
 \STATE {\bf Input:} Budget $B$; $((\vc{x}_1,y_1),\ldots,(\vc{x}_m,y_m)) \in \reals^{dk+1}$,
 Algorithm type: Scoring/Full.
 \STATE {\bf Output:} A repeat vector $\br \in R_B$.
 \STATE $\bar{x}_i[a] \leftarrow \frac{1}{\rtrain}\sum_{j \in [\rtrain]} x_i[a](j)$ for $i \in [m], a\in A$.
 \STATE $\vc{\hat{\cov}} \leftarrow \frac{1}{m}\sum_{i} y_i \vc{\bar{x}_i}$.
 \STATE  $\hat{v}[a] \leftarrow \frac{1}{m}\sum_{i}\mathrm{VarEst}(x_i[a](1),\ldots,x_i[a](\rtrain))$.
 \IF{Scoring Algorithm}
 \STATE $\forall a \in A,\hat{\sigma}^2[a] \leftarrow \max\left\{0, \frac{1}{m}\sum_{i} (\bar{x}_i[a])^2-\frac{\hat{v}[a]}{k}\right\}$.\label{step:k}
  \STATE Define $\hat{\obj}(\br) \equiv \sum_{a: r[a]>0} \hat{\cov}[a]^2/(\hat{\sigma}^2[a] + \frac{\hat{v}[a]}{r[a]})$
  \ELSE
 \STATE $\hat{\Sigma} \leftarrow \mathrm{MakePSD}\left(\frac{1}{m}\sum_i \vc{\bar{x}}_i^T\vc{\bar{x}}_i - \Diag(\vc{\hat{v}})/k)\right)$\label{step:vk}
 \STATE $\mt{M}_\br \equiv \submat_\br(\hat\Sigma+\Diag(\frac{\hat{v}[1]}{r[1]},\ldots,\frac{\hat{v}[d]}{r[d]}))$
\STATE Define $\hat{\obj}(\br) \equiv \submat_\br(\hat{\vc{\cov}})^T\mt{M}_\br^+\submat_\br(\hat{\vc{\cov}})$
\ENDIF
\STATE $\br_0 \leftarrow (0,\ldots,0) \in \nats^d$
\FOR {$t=1$ to $B$}
\STATE Find $i_\mathrm{best} \in [d]$  such that $\hat{\obj}(\br_{t-1} + \vc{e}_i)$ is maximal.\\
\STATE $\br_t \leftarrow \br_{t-1} + \vc{e}_{i_\mathrm{best}}$.
\ENDFOR
\STATE Return $\br_B$.
\end{algorithmic}
\end{algorithm}

\subsection{Guarantees}

Under our distributional assumptions, we show that the estimated objective functions used by our algorithms converge to $\E[Y^2] - \loss(\br)$. Thus maximizing the estimated objective approximately minimizes $\loss(\br)$.
Formally, let $\hat{\obj}_f(\br)$ and $\hat{\obj}_s(\br)$ be the objectives used in \algref{alg:all} for the full algorithm and the Scoring algorithm, respectively. Note that these objectives are implicitly functions of the training sample $S$.
For a symmetric matrix $\mt{S}$, let $\lambda_{\min}(\mt{S})$ be the smallest eigenvalues of $\mt{S}$. We define:
$\lambda = \min_{\br \in R_B} \lambda_{\min}(\submat_\br(\Sigma))$, and $\bar{B} = \min(B,d)$.
\begin{theorem}\label{thm:convergence}
Suppose that all judgments and labels are in $[-1,1]$.  
Then for any $\delta \in (0,1)$, with prob.\ at least $1-\delta$ over $m$ i.i.d.\ training samples from $D_\brtrain$, for all $\br \in R_B$,
for $m \geq \tilde{\Omega}(\bar{B}\ln(\bar{B}d/\delta)/\lambda^2)$ we have
\[
|\hat{\obj}_f(\br)-(\E[Y^2]-\loss(\br))| \leq O\left(\frac{\bar{B}^3\ln(Bd/\delta)}{\lambda^{2}\sqrt{m}}\right).
\]
If the external covariance matrix $\Sigma$ is diagonal, then
for $m \geq \tilde{\Omega}(\ln(d/\delta)/\lambda^2)$ we have 
\begin{align*}
&|\hat{\obj}_s(\br)-(\E[Y^2]-\loss(\br))| \leq O\left(\frac{\ln(Bd/\delta)}{\lambda^{2}\sqrt{m}}\right).
\end{align*}
\end{theorem}
\arxiv{The proof of this theorem is provided in \appref{app:analysisproofs}.}
The convergence rate for the full algorithm stems from two bounds: (1) If the norm of the minimizing $\vc{w}$ is at most $\alpha$, then the convergence rate is at most $\bar{B}\alpha^2/\sqrt{m}$; (2) With high probability, the norm of the minimizing $\vc{w}$ is at most $\sqrt{\bar{B}}/\lambda$. An additional factor of $O(\bar{B} \ln (Bd))$ gets uniform convergence over $\br \in R_B$.
The components of this result are of the same order as the equivalent results for uniform convergence of standard least-squares regression. An improved rate of $\sqrt{\bar{B}\alpha^2 /m}$ can be achieved for least-squares regression, \emph{if the algorithm exactly minimizes the sample squared loss} \citep{SrebroSrTe10}. However, our algorithm minimizes another objective, thus this result is not directly applicable. We leave it as a challenge for future work to find out whether a faster rate can be achieved in our case.

As always, these convergence rates are worst-case, and in practice a much smaller sample size is often sufficient to get meaningful results, as we have observed in our experiments. 
However, if the available training sample is too small to achieve reasonable results, one can limit the norm of the minimizer by adding regularization to the estimated covariance matrix, as in ridge regression \citep{HoerlKe70}. This would allow faster convergence at the expense of a more limited class of predictors.

As \thmref{thm:convergence} shows, when the zero-correlation assumption holds, the Scoring algorithm enjoys a much faster worst-case rate of convergence than the full algorithm. This is because it does not attempt to estimate the entire covariance matrix. This advantage is more significant for larger budgets. An additional advantage is that it finds the \emph{optimal} value of $\br$ for its estimated objective:
\begin{theorem}\label{thm:greedy}
The Scoring algorithm returns $\br \in \argmax_{\br \in R_B} \hat{\obj}_s(\br)$.
\end{theorem}
\thmref{thm:greedy} follows since $f(r) = a/(b+c/r)$ is concave and increasing in $r$ and due to the following observation.
\begin{lemma}\label{lem:generalgreedy}
Let $\br \in \nats^d$, and let $f(\br) = \sum_{i \in [d]} g_i(r[i])$,
where $g_i(\cdot):\reals_+\rightarrow \reals$ are monotonic non-decreasing concave functions. Let $B \in \nats$.
The maximum of $f(\br)$ subject to $\br \in R_B$ is attained by a greedy algorithm which starts with $\br = (0,\ldots,0)$, and iteratively increases the coordinate which increases $f$ the most.
\end{lemma}
\arxiv{The proof of this lemma is provided in \appref{app:analysisproofs}.}

\subsection{Implementation}
If our estimate of $\Sigma$ is not PSD, we use the procedure `MakePSD', which takes a symmetric matrix $\mt{A}$ as input, and returns the PSD closest to $\mt{A}$ in Frobenius norm. This can be done by calculating the eigenvalue decomposition $\mt{A} = \mt{U}\mt{D}\mt{U}^T$ where $\mt{U}$ is orthogonal and $\mt{D}$ is diagonal, and returning $\mt{U}\tilde{\mt{D}}\mt{U}^T$, where $\tilde{\mt{D}}$ is $\mt{D}$ with zeroed negative entries \citep{Higham88}. If we assume a diagonal external covariance, then this procedure is equivalent to rounding up the estimate of $\sigma^2(a)$ to zero, as done in the Scoring algorithm.
For a budget of $B$, the full algorithm performs $Bd$ SVDs to calculate pseudo-inverses. Note, however, that the largest matrix that might be decomposed here is of size $\min(d,B)\times \min(d,B)$. Furthermore, in practice the matrices can be much smaller, since the algorithm might choose several repeats of the same features. In our experiments, the total time for decompositions, using standard libraries on a standard personal computer, has been negligible.

Our description of the algorithms above assumes for simplicity that the mean of all features is zero. In practice, one adds a `free' feature that is always 1, to allow for biased regressors. For the Scoring algorithm, one should further subtract the empirical mean from each feature. For the full algorithm, this not necessary, because when bias is allowed, adding a constant to any feature provably will not change the output of the full algorithm.

\section{Experiments}\label{sec:experiments}
\begin{figure}[b]
\fbox{\includegraphics[width = 0.95\linewidth]{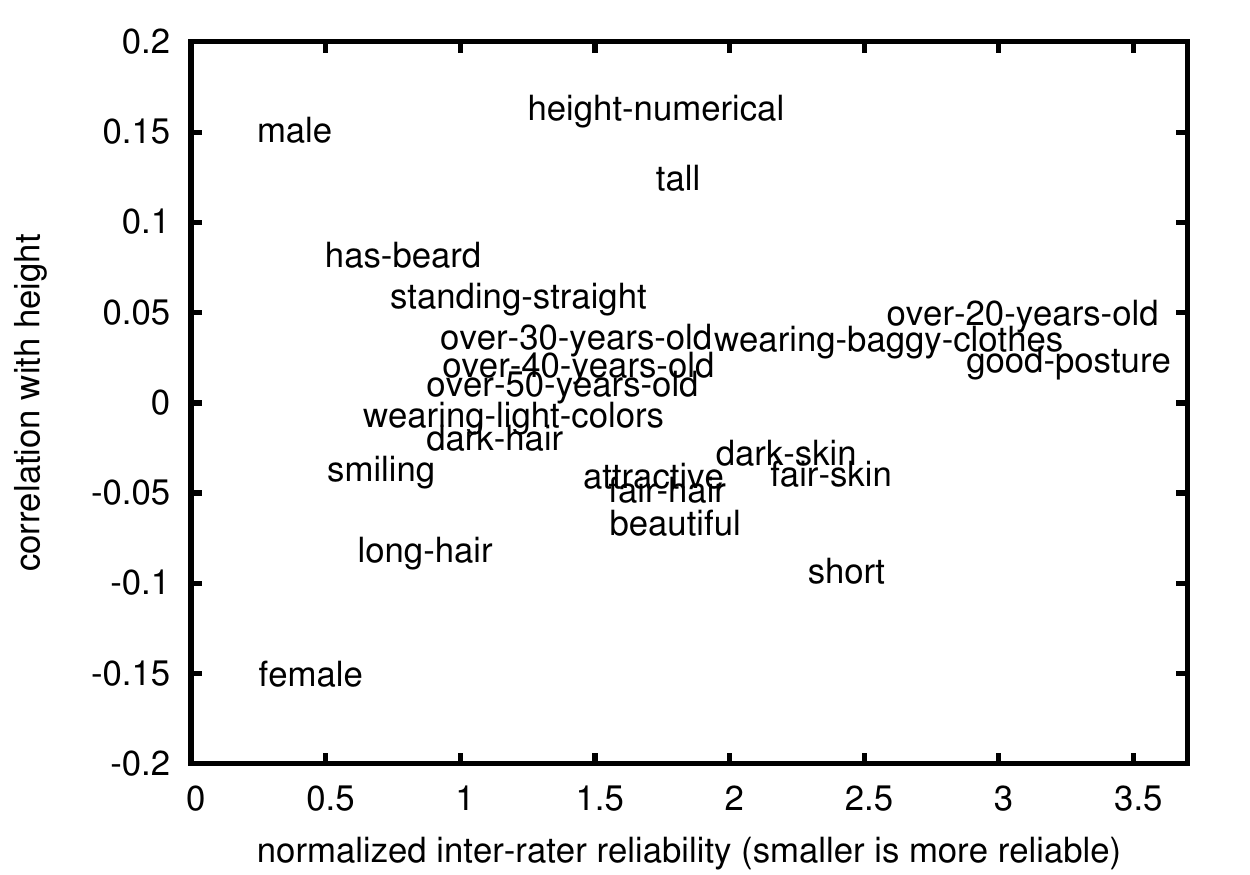}}
\caption{Properties of selected attributes for height prediction}
\label{fig:heightatt}
\end{figure}
\begin{figure}[b]
\fbox{\includegraphics[width = 0.95\linewidth]{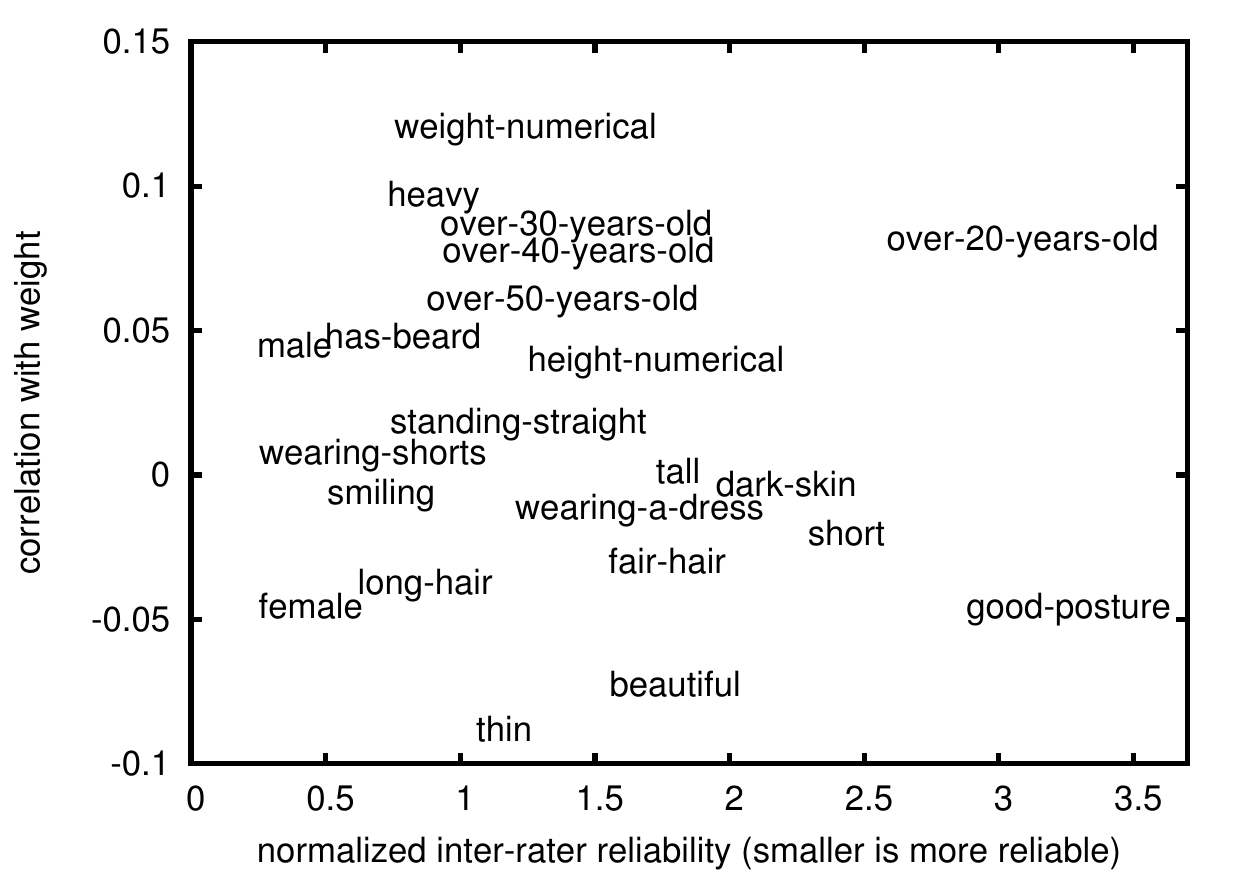}}
\caption{Properties of selected attributes for weight prediction}
\label{fig:weightatt}
\end{figure}

We tested our approach on three regression problems. In the first problem the feature judgments were simulated. In the second and third problem they were collected from the crowd using Amazon's Mechanical Turk.\footnote{{\tt http://mturk.com}.  We will share our data upon request from other researchers, due to the sensitivity of judgments on people's images.}

For the simulated experiment we used the UCI dataset `Relative location of CT slices on axial axis Data Set' \citep{UCI10}. In this dataset the features are histograms of spatial measurements in the image, and the label to predict is the relative location of the image on the axial axis.
To simulate features with varying judgments, we collapsed each set of 8 adjacent histogram bins into a single feature, so that each judgment of the new feature was randomly chosen out of 8 possible values for this feature. The resulting dataset contained 48 noisy real-valued features per example. 

The second and third problems were to predict the height and weight of people from a photo.
880 photos with self-declared height and weight were extracted from the publicly available \href{http://www.cockeyed.com/photos/bodies/heightweight.html}{Photographic Height/Weight Chart} \cite{HeightWeight}, where people post pictures of themselves announcing their own height and weight. We chose 37 attributes that we felt the crowd could judge and might be predictive. %\todo{what is good work? what ground truth did we use?}
We collected judgments for these binary attributes, mainly following the judgment collection methodology of \citet{PH12}, by batching the images into groups of 40, making labeling very efficient. To encourage honest workers, we promised (and delivered) bonuses for good work. We further limited the amount of work any one person could do. We used all of the collected judgments, regardless of whether the workers received bonuses for them or not. Our pay per hour was set to average to minimum wage. We collected numerical estimates of the height and the weight in a similar fashion. Binary judgments took about one second per judgment and their cost was a fraction of a cent per attribute judgment. The numerical estimates took about four times as long and we paid four times as much for them. Accordingly, we adjusted all the algorithms to count a single numerical judgment as equal to four binary attribute judgments.

\figref{fig:heightatt} and \figref{fig:weightatt} show the normalized correlation ($\hat{\cov}[a]/\hat{\sigma}[a]$) vs. the normalized inter-rater reliability ($\hat{v}[a]/\hat{\sigma}[a]$) of selected attributes. These plots demonstrate that all combinations of useful/non-useful and stable/noisy attributes exist in this data. The full data listing all the attributes and their properties is provided in \tabref{tab:atts}.

\begin{table}
\begin{filecontents*}{attribute_corrs.csv}
attribute,var.,height,weight
male,0.39,0.15,0.04
female,0.45,-0.15,-0.05
outdoors,0.45,0.00,-0.01
wearing-glasses,0.52,-0.01,0.02
near-water,0.61,0.00,-0.01
wearing-long-sleeves,0.63,0.02,0.01
wearing-shorts,0.68,0.01,0.01
looking-at-the-camera,0.70,-0.01,0.00
smiling,0.71,-0.04,-0.01
has-beard,0.79,0.08,0.05
long-hair,0.87,-0.08,-0.04
facing-sideways-(profile),0.89,0.01,-0.01
hands-on-hips,0.90,-0.04,0.00
heavy,0.90,-0.02,0.10
dark-hair,1.13,-0.02,-0.02
thin,1.16,0.01,-0.09
wearing-dark-colors,1.17,-0.01,0.01
wearing-light-colors,1.20,-0.01,-0.01
can-see-a-lot-of-skin,1.21,-0.05,-0.03
standing-straight,1.22,0.06,0.02
weight (numerical),1.24,0.04,0.12
wearing-colorful-clothes,1.36,-0.04,-0.01
over-50-years-old,1.38,0.01,0.06
over-30-years-old,1.43,0.04,0.09
over-40-years-old,1.44,0.02,0.08
wearing-fancy-clothes,1.56,-0.06,-0.02
wearing-a-dress,1.67,-0.05,-0.01
wearing-nice-clothes,1.69,-0.05,-0.03
attractive,1.72,-0.04,-0.07
height (numerical),1.73,0.16,0.04
fair-hair,1.77,-0.05,-0.03
beautiful,1.80,-0.07,-0.07
tall,1.81,0.12,0.00
dark-skin,2.21,-0.03,0.00
fair-skin,2.38,-0.04,-0.02
short,2.43,-0.09,-0.02
wearing-baggy-clothes,2.59,0.03,0.07
over-20-years-old,3.09,0.05,0.08
good-posture,3.26,0.02,-0.05
\end{filecontents*}
\csvautotabular{attribute_corrs.csv}
\caption{All attributes used for height and weight prediction. `var' indicates the normalized inter-rater variability ($\hat{v}[a]/\hat{\sigma}[a]$). `height'  and `weight' indicate the estimated quality of each of the attributes for the respective prediction task ($\hat{\cov}[a]/\hat{\sigma}[a]$).}
\label{tab:atts}
\end{table}

\tabref{tab:atts} Lists all the attributes that were collected for the height and weight prediction problem, their internal variance and their normalized correlation for each of the prediction tasks.

We compared the test error of our algorithms, denoted `Full' and `Scoring' in the plots, to those of several plausible baselines. In all comparisons, we set $\rtrain = 2$. The first baseline, denoted `Averages' in the plots, is based on the ``predictive'' feature selection algorithm of \citet{Memorability}: We first average the 2 judgments per attribute to create a standard data set with one value for each object-attribute pair, and then greedily add attributes, one at a time, so as to minimize the least-squares error. The resulting regressor uses $2$ judgments for each selected feature.
The second baseline, denoted `Copies', treats the 2 judgments of each feature-object pair as 2 different individual attributes, and again performs greedy forward selection on these features. Here the test repeat vector $\br$ was set according to the number of copies selected for each feature. Note that these baselines perform standard Machine Learning feature selection: `Averages' considers $d$ features and `Copies' considers $2d$ features.
For height and weight prediction, we compared the results also to the test error achieved by averaging only the height or weight estimates of the crowd, respectively. Since each numerical feature costs 4 times as much as a binary feature, we averaged over $B/4$ numerical judgments when the budget was set to $B$. We did not use regularization anywhere, thus our algorithms and the baselines are all parameter-free.

The test error presented in the plots was obtained as follows: $\br$ was selected based on a training set with $\brtrain$ judgments. We then added judgments to features in the training set to get to $\br$ repeats. Finally we performed regular regression on the means of the enhanced training set to get a predictor. This predictor was then used to predict the labels of the test set with $\br$ judgments. In all the comparisons, each experiment was averaged over 50 random train/test splits. In all of the experiments, shown in figures \ref{fig:slice}-\ref{fig:w3}, our full algorithm achieved better test error than the baselines. The Scoring algorithm was usually somewhat worse than the Full Multi-Selection algorithm,
and for small budgets also sometimes worse than the baselines, This is expected due to its zero-correlation assumption. However, when the sample size was small, the Scoring algorithm was sometimes better (see e.g., \figref{fig:h3}), since it suffered from less over-fitting. This is consistent with our convergence analysis in \thmref{thm:convergence}. Analysis of training errors indicates that baseline algorithms suffer for two different reasons: (1) they are limited to a small number of repeats per feature; and (2) they suffer from greater over-fitting. The second reason is probably due to the fact that our algorithm tends to select a sparser $\br$ than do the baselines. Table \ref{tab:pred} shows examples of predictors, with number of judgments for each attribute, learned by our full algorithm.

\begin{figure}[h]
\includegraphics[width = 0.95\linewidth]{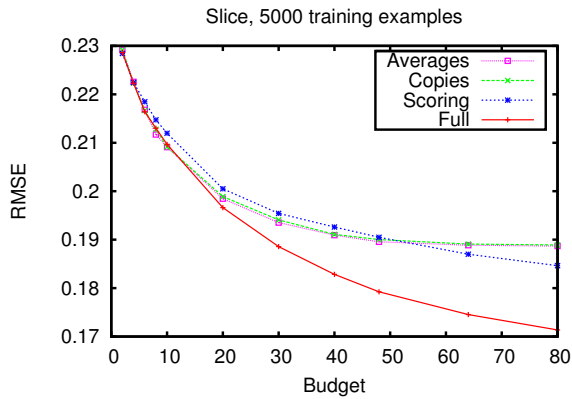}
\caption{Comparison for the `Slice' dataset.
\vspace{-.2in}}
\label{fig:slice}
\end{figure}
\begin{figure}[h]
\includegraphics[width = 0.95\linewidth]{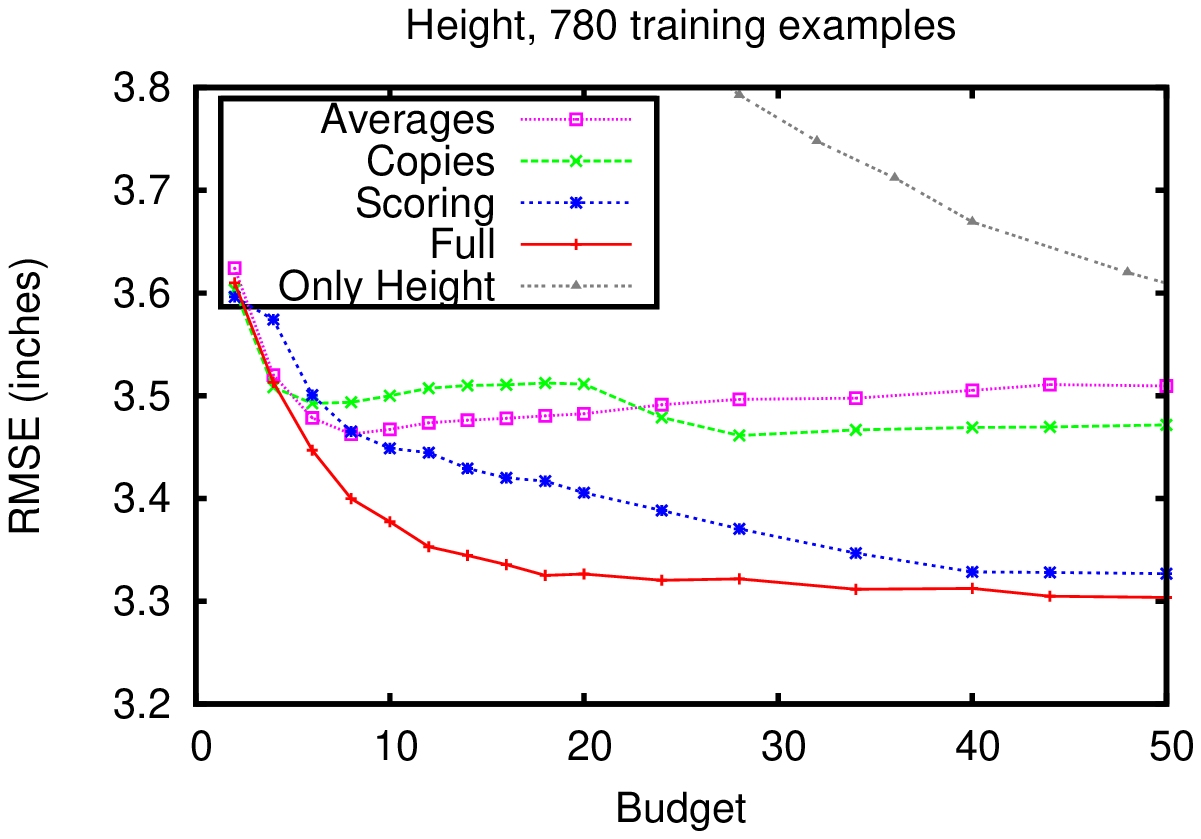}
\caption{Comparing algorithms}
\label{fig:h}
\end{figure}

\begin{figure}[h]
\includegraphics[width = 0.95\linewidth]{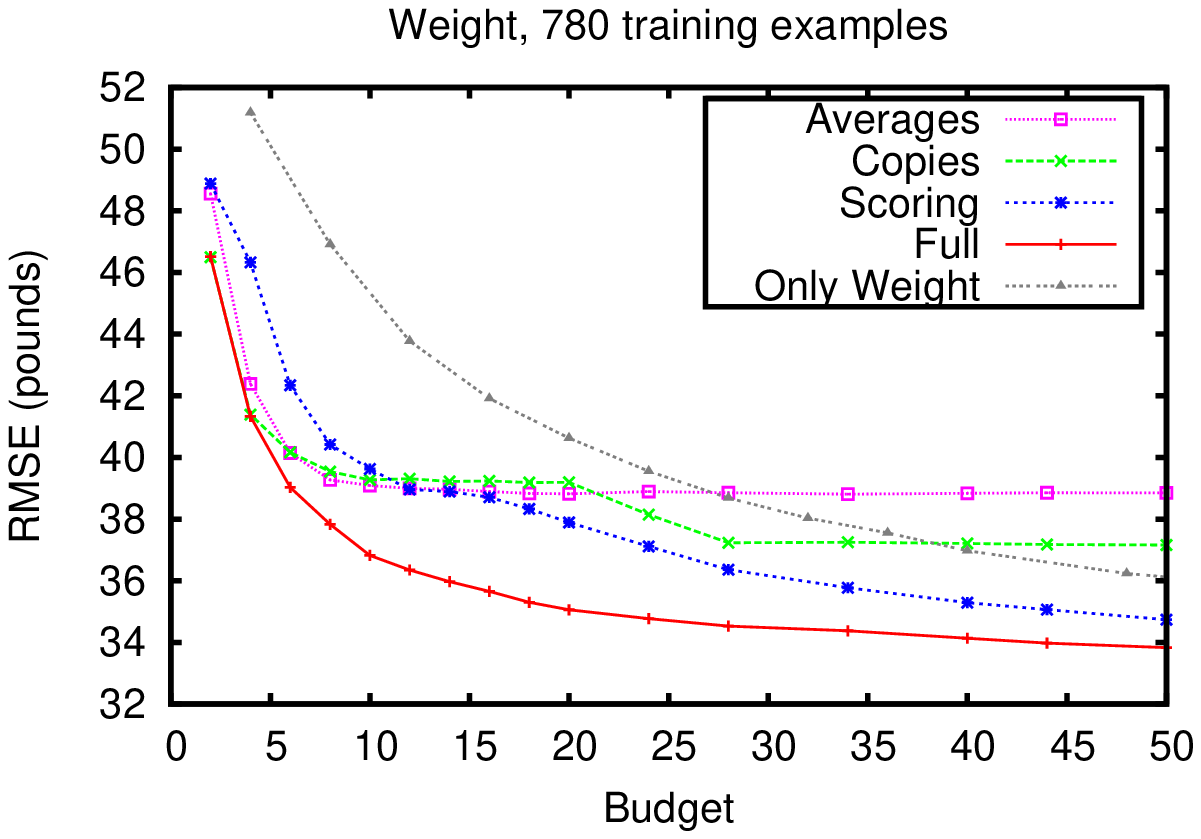}
\caption{Comparing algorithms}
\label{fig:w}
\end{figure}
\begin{figure}[h]
\includegraphics[width = 0.95\linewidth]{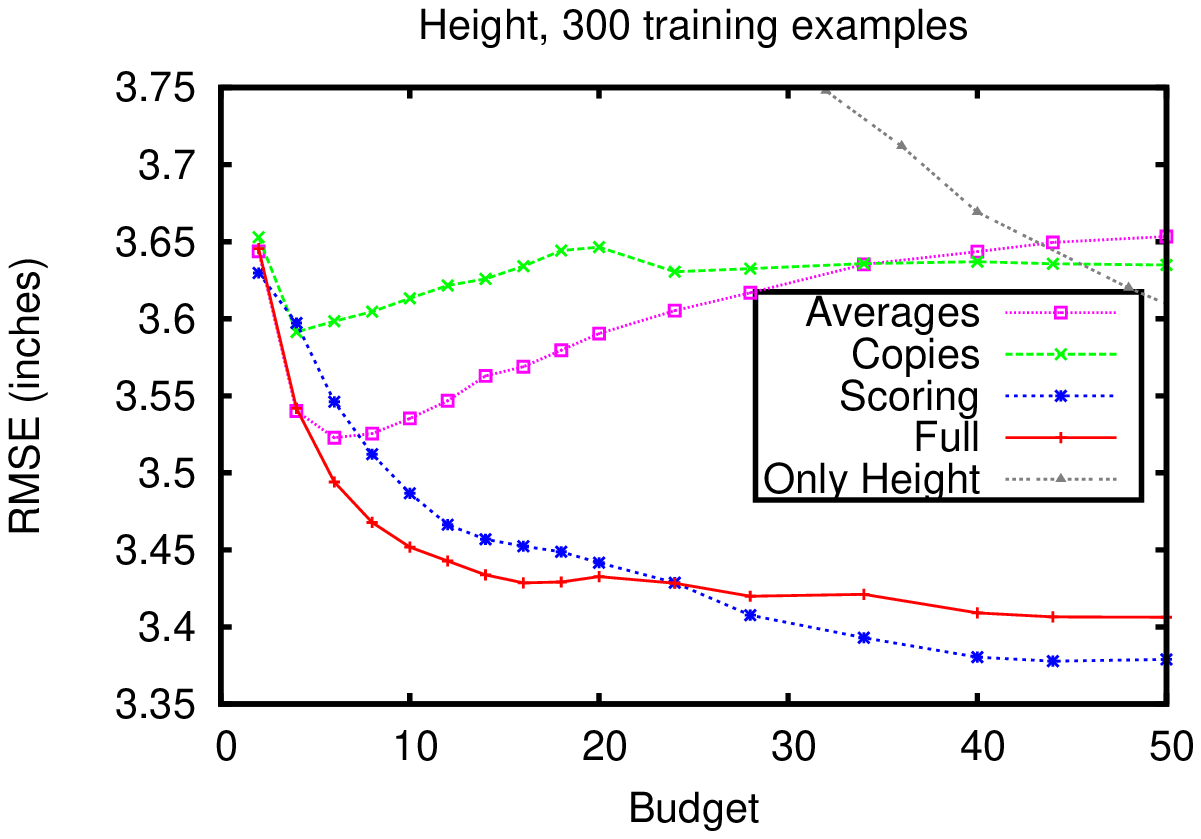}
\caption{Comparing algorithms}
\label{fig:h3}
\end{figure}
\begin{figure}[h]
\includegraphics[width = 0.95\linewidth]{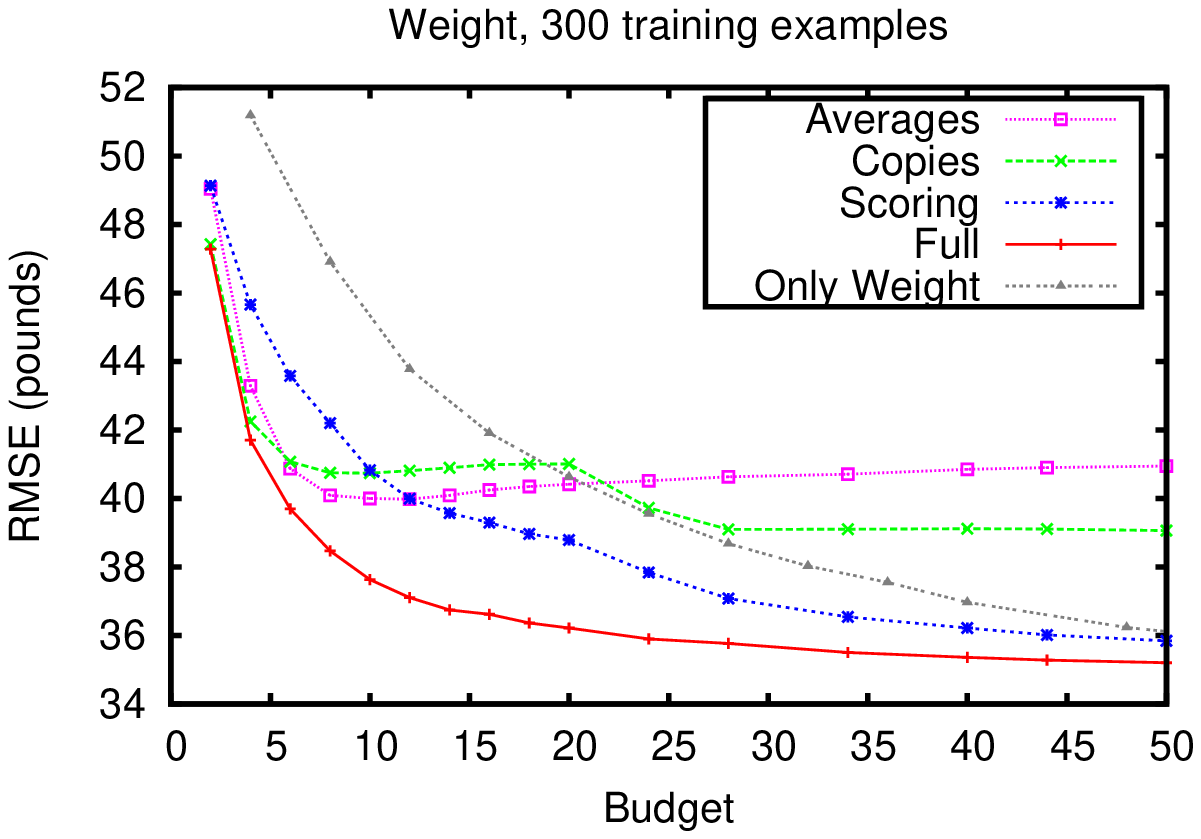}
\caption{Comparing algorithms}
\label{fig:w3}
\end{figure}
\begin{table}[h]
\center
Height (inches)\\
{\small{
\begin{filecontents*}{height_pred.csv}
Attribute,reps,coefficient
dark hair,1,-0.39
dark skin,1,-0.52
female,1,-0.94
heavy,1,0.61
long hair,1,-0.59
looking at the camera,1,-0.45
male,5,4.88
outdoors,1,0.55
tall,17,7.38
wearing baggy clothes,6,-2.49
(bias),N/A,58.82
\end{filecontents*}
\csvautotabular{height_pred.csv}
%}}
%\caption{Examples of predictors generated by the full algorithm for prediction of height (top) and weight (bottom). For each attribute, the coefficient is multiplied by the average of this attribute over the selected number of judgment repeats.}
%\label{tab:pred}
%\end{table}
%\begin{table}[b]
%{\small{
\vspace{1em}

Weight (pounds)\\
\begin{filecontents*}{weight_pred.csv}
Attribute,reps,coefficient
attractive,2,10.08
can see a lot of skin,1,-6.44
female,3,-36.15
has beard,1,14.47
heavy,17,125.33
long hair,1,-8.06
looking at the camera,1,-6.78
male,1,10.09
over 30 years old,14,59.58
over 40 years old,7,-34.37
thin,2,-12.80
(bias),N/A,-87.56
\end{filecontents*}
\csvautotabular{weight_pred.csv}}}
\caption{Examples of predictors generated by the full algorithm for prediction of height and weight. For each attribute and image, the selected number of judgment repeats is collected, and the coefficient of the attribute is multiplied by the fraction of judgments that designated this attribute as true for the image.}
\label{tab:pred}
\end{table}

\begin{figure}[h]
\includegraphics[width = 0.95\linewidth]{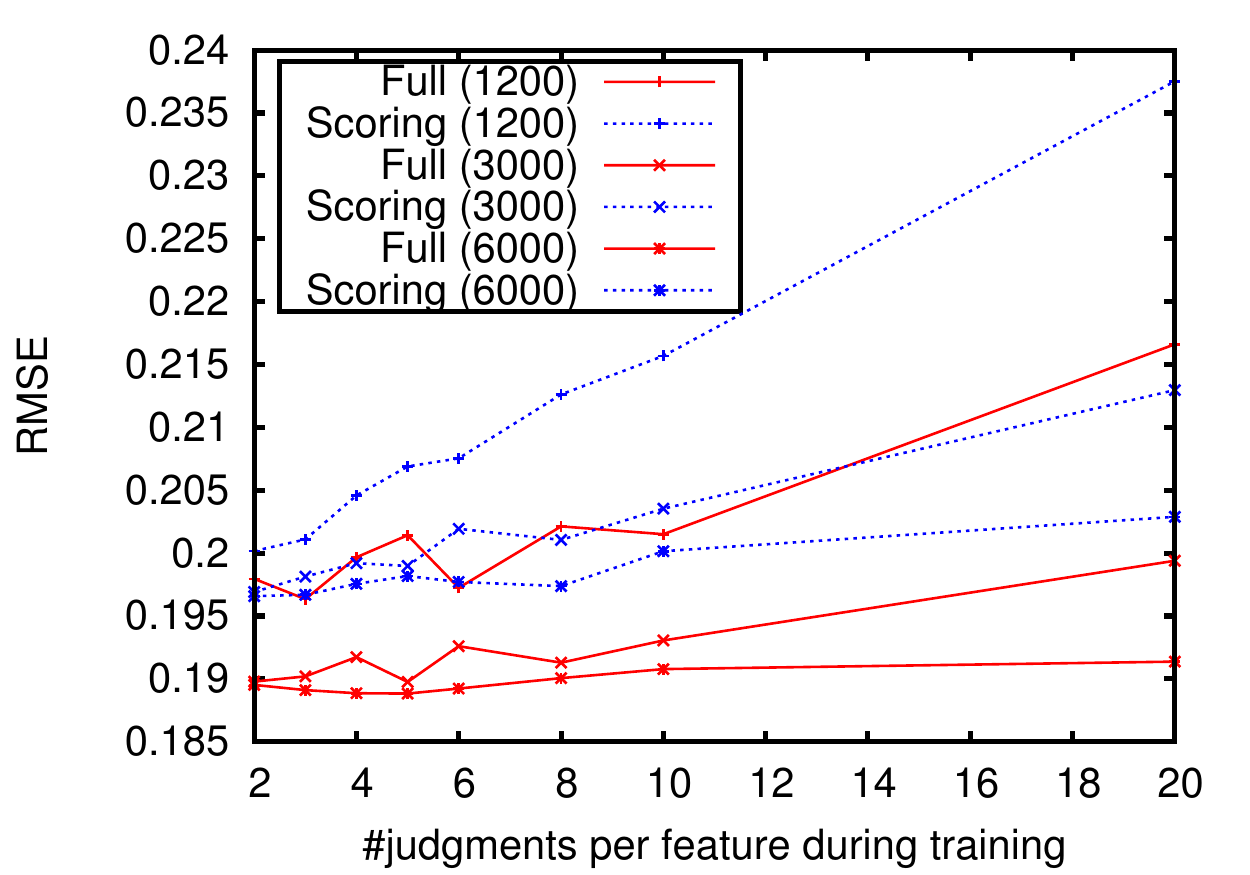}
\caption{`Slice' dataset: Training loss with different numbers of judgments per feature during training, when $\#$training examples$\times\#$repeats is kept constant. Numbers in legend indicate value of constant. }
\label{fig:tradeoff}
\end{figure}

In our last experiment we tested the tradeoff between the number of training judgments per feature, and the number of training examples, in the following setting: Suppose we have a budget that allows us to collect a total of $M$ judgments for training the feature multi-selection algorithm, and we have access to at least $M/2d$ labeled examples. We can decide on a number $\rtrain$ of judgments per feature, randomly select $M/\rtrain d$ objects from our labeled pool to serve as the training set, and obtain $\rtrain d$ judgments for each of these objects. What number $\rtrain$ should we choose? Does this number depend on the total budget $M$? 
We compared the test error arising from different values of $\rtrain$ over different values of $M$, for the slice dataset using both of our algorithms,.
The results are shown in \figref{fig:tradeoff}. These results show a clear preference for a small $\rtrain$ (which allows a large $m$ on the same budget $M$). Characterizing the optimal number $\rtrain$ is left as an open question for future work.

\section{Conclusions}

We introduce the problem of \emph{feature multi-selection}
and provide two algorithms for the case of regression with mean averaging of judgments.
Future directions of research include other learning tasks, such as classification,
and other types of feature aggregation, such as median averaging (which, for binary features,
is equivalent to taking the majority). An additional important question for future work is
how to carry out feature multi-selection in an environment with a changing crowd.

\section*{Acknowledgments}
 We wish to thank Edith Law and Haoqi Zhang for several helpful discussions.


%\clearpage

\bibliography{bib}
\bibliographystyle{icml2013}

%\clearpage

\appendix

\iffalse
\begin{center}
\hrule
{\Large{\textbf{\ourtitle}}}
\hrule
\end{center}
\fi

\section{Analysis}\label{app:analysisproofs}

\subsection{Notations}

For a symmetric matrix $\mt{S}$, let $\lambda_{\max}(\mt{S})$ and $\lambda_{\min}(\mt{S})$ be the largest and the smallest eigenvalues of $\mt{S}$, respectively. For functions $\alpha$ and $\beta$, 
we say that $\alpha \leq O(\beta)$
 if for some constants $C,C'>0$, $\alpha \leq C\beta  + C'$. Similarly, $\alpha \geq \Omega(\beta)$ indicates that for some constants $C,C'>0$, $\alpha \geq C\beta - C'$.
we say that $\alpha \leq \tilde{O}(\beta)$
 if for some constants $C,C'>0$, $\alpha \leq C\beta  \ln (\beta) + C'$. Similarly, $\alpha \geq \tilde{\Omega}(\beta)$ indicates that for some constants $C,C'>0$, $\alpha \geq C\beta  \ln (\beta) - C'$.

Denote by $\mt{Z}_\br$ a diagonal $d\times d$ matrix whose $i$'th diagonal entry is $\one[r[i] > 0]$.
Recall that $v[i]$ is the internal variance of feature $i$. For a vector $\br \in \nats^d$, let $\mt{V}_\br$ be the diagonal matrix such that its $i$'th diagonal entry is zero if $r[i] = 0$ and equal to $v[i]/r[i]$ otherwise. Let $\hat{\mt{V}}_\br$ be defined similarly but using the sample estimate $\hat{v}[i]$, defined in \eqref{eq:hatv}, instead of $v[i]$.
Denote by $n_\br$ the number of non-zero entries in $\br \in \nats^d$.

Let $\vc{x} \in \reals^{[\brtrain]}$ be an example with a repeat vector $\brtrain$ such that $k[i] \geq 2$.\footnote{\algref{alg:all} can be applied to $\brtrain$ with different values per feature by simply using $k[a]$ instead of $k$ in step \ref{step:k} and $\mt{V}_\brtrain$ instead of $\Diag(\hat{\vc{v}})/k$ in step \ref{step:vk}. Our analysis holds for this more general algorithm.}
Let $\hat{v}(\vc{x})[i] = \mathrm{VarEst}(x[i](1),\ldots,x[i](\rtrain[i]))$. Given a repeat vector $\br$, define the $\br$-loss $\loss_\br$ of the labeled example $(\vc{x}, y)$ by:
\begin{align*}
&\loss_\br(\vc{w},\vc{x}, y)=\\ 
&\quad(\dotprod{\mt{Z}_\br\vc{w},\vc{\bar{x}}} - y)^2 + \sum_{i:r[i] > 0}w[i]^2 \hat{v}(\vc{x})[i](\frac{1}{r[i]} - \frac{1}{\rtrain[i]}).
\end{align*}

Let $S = ((\vc{x}_1,y_1),\ldots,(\vc{x}_m,y_m))$ be a training set of labeled representations drawn i.i.d.\ from $D_\brtrain$. Denote the vector of training labels by $\vc{y} = (y_1,\ldots,y_l)$. We denote the average of $\loss_\br$ over $S$ by
\[
\loss_\br(\vc{w}, S) = \frac{1}{m}\sum_{l \in [m]}\loss_\br(\vc{w}, \vc{x}_l, y_l).
\]
Define
\begin{align*}
\hat{\Sigma} &= \frac{1}{m}\sum_{l\in[m]}\vc{\bar{x}}_l \vc{\bar{x}}^T_l - \mt{\hat{V}}_\brtrain,\\
\hat{\Sigma}_{\br} &=\mt{Z}_{\br}(\hat\Sigma + \mt{\hat{V}}_{\br})\mt{Z}_{\br},\\
\hat{\vc{\cov}} &= \frac{1}{m}\sum_{l \in [m]} y_l \bar{\vc{x}}_l.
\end{align*}
Note that the notation for $\hat\Sigma$ here is \emph{different than the one used in \algref{alg:all}}, since $\hat{\Sigma}$ is \emph{not} `corrected' to be PSD.
We denote the corrected estimate used in the algorithm by $\hat{\Sigma}^p$. Similarly, we denote by $\hat{\Sigma}_\br^p$ the estimate for $\Sigma_\br$ resulting from using $\hat{\Sigma}^p$ instead of $\hat{\Sigma}$.
We have
\[
\loss_\br(\vc{w}, S) = \vc{w}^T \hat{\Sigma}_{\br} \vc{w} - 2 \vc{w}^T \mt{Z}_\br\vc{\hat{\cov}} + \frac{1}{m}\vc{y}^T\vc{y}.
\]
Define $\lambda_\br = \lambda_{\min}(\submat_\br(\Sigma))$. Note that for $\lambda$ defined in the statement of \thmref{thm:convergence}, $\lambda = \min_{\br \in R_B} \lambda_\br$.

\subsection{Proof of \thmref{thm:convergence}} \label{sec:convanalysis}
To prove \thmref{thm:convergence} we require several lemmas. All of the analysis below is under the assumption of \thmref{thm:convergence}, that all judgments and labels are in $[-1,1]$ with probability 1.

First, the following lemma links
the covariance matrix of the population when averaging $\br$ judgments to the covariance matrix of the population when using the expected values of the features for each object.
\begin{lemma}\label{lem:sigma}
Let $\br \in \nats^d$. Then $\Sigma_{\br} = \mt{Z}_\br\Sigma\mt{Z}_\br + \mt{V}_\br.$
\end{lemma}
\begin{proof}
Consider a random object $O$ drawn from $\cD_O$, and define the vector $\mu_O$ such that for each coordinate $i$,
$\mu_O[i] = \E[X[i]|O]$. Define $\Sigma_{O} = \mu_O \mu^T_O$, so that $\Sigma = \E_{O \sim \cD_O}[\Sigma_{O}]$. 
Similarly, define $\Sigma_{\br|O} = \E_\br[\vc{\bar{X}}_\br^T \vc{\bar{X}} \mid O]$, so that $\Sigma_{\br} = \E_{O \sim \cD_O}[\Sigma_{\br|O}]$.
Lastly, denote the variance of a feature on object $O$ by $v_O[i]= \E[(X[i] - \E[X[i]|O])^2|O]$,
and let $\mt{V}_{\br|O}$ be the diagonal matrix whose $i$'th entry is zero if $r[i] = 0$,
and equal to $v_O[i]/r[i]$ otherwise. Again $\mt{V}_\br = \E_O[\mt{V}_{\br|O}]$.
We will prove that for any object $O$, 
\begin{equation}\label{eq:sigmae}
\Sigma_{\br|O} = \mt{Z}_\br\Sigma_{O}\mt{Z}_\br + \mt{V}_{\br|O}.
\end{equation}
The desired equality will follow by averaging over $O \sim \cD_O$. 

Denote the entries of $\Sigma_{\br|O}$ by $s_{ik}$. First, whenever $r[i] = 0$ or $r[k] = 0$, entry $(i,k)$ is zero for both sides of \eqref{eq:sigmae}. Now, consider a non-diagonal entry $s_{ik}$ of $\Sigma_{\br|O}$ for $i \neq k$, $r[i] >0$ and $r[k] > 0$. The $(i,k)$ entry on the right-hand side of \eqref{eq:sigmae} is $\mu_O[i]\mu_O[k]$. For the left-hand side we have 
\begin{align*}
s_{ik} &= \E_\br[\bar{X}[i]\bar{X}[k]\mid O] = \E_\br[\bar{X}[i]\mid O]\cdot \E_\br[\bar{X}[k]\mid O]\\
&= \mu_O[i]\mu_O[k].
\end{align*}
Thus the equality holds for non-diagonal entries.

Now, consider a diagonal entry $s_{ii}$ of $\Sigma_{\br}$. If $r[i] = 0$ then both sides of \eqref{eq:sigmae} are zero. 
Assume $r[i] > 0$. We have
\begin{align*}
s_{ii} &= \E_\br[\bar{X}[i]^2 \mid O] = \E_\br[(\frac{1}{r[i]}\sum_{j \in [r[i]]} X[i](j))^2 \mid O] \\
&= \frac{1}{r[i]^2}\big(\sum_{j \in [r[i]]}\E_\br[(X[i](j))^2 \mid O] +\\ 
&\quad2\cdot\sum_{j < j', j,j' \in [r[i]]}\E_\br[X[i](j)\cdot X[i](j') \mid O]\big).
\end{align*}
Conditioned on $O$, $X[i](j)$ and $X[i](j')$ are statistically independent.  Therefore
\begin{align*}
&\E_\br[X[i](j)\cdot X[i](j')\mid O] \\&\quad= \E_\br[X[i](j) \mid O]\cdot \E[X[i](j') \mid O] = \mu_{O}[i]^2.
\end{align*}
In addition, $\E_\br[(X[i](j))^2] = v_O[i] + \mu_{O}[i]^2$. Combining the two equalities we get 
\begin{align*}
s_{ii} &= \frac{1}{r[i]}(v_O[i] + \mu_{O}[i]^2) + \frac{1}{r[i]^2}\cdot(r[i]^2 - r[i]) \mu_{O}[i]^2\\
&= \mu_{O}[i]^2 + v_O[i]/r[i].
\end{align*}
This is exactly the value of entry $(i,i)$ on the right-hand side of \eqref{eq:sigmae}.
We conclude that \eqref{eq:sigmae} holds for all types of entries, thus the lemma is proved.
\end{proof}

Using \lemref{lem:sigma}, we can show that $\loss_\br(\vc{w},S)$ is an unbiased estimator of $\loss(\vc{w}, D_{\br})$.
\begin{lemma}\label{lem:rdiff}
Let $\brtrain,\br \in \nats^d$, so that  $\forall i \in [d], k[i] \geq 2$. Let $S$ be a sample drawn i.i.d.\ from $D_\brtrain$.
For any $\vc{w} \in \reals^d$, $\loss(\vc{w}, D_{\br}) = \E_S[\loss_\br(\vc{w},S)].$
\end{lemma}

\begin{proof}
We have 
\begin{align*}
 &\E_S[\loss_\br(\vc{w},S)]= \E_S[\vc{w}^T \hat{\Sigma}_{\br} \vc{w} - 2 \vc{w}^T \mt{Z}_\br \vc{\hat{\cov}} + \frac{1}{m}\vc{y}^T\vc{y}] \\
&\quad=\vc{w}^T \E_S[\hat{\Sigma}_{\br}] \vc{w} - 2 \vc{w}^T \mt{Z}_\br \E_S[\vc{\hat{\cov}}] + \E_S[\frac{1}{m}\vc{y}^T\vc{y}].
\end{align*}
From the definition of $\hat{\Sigma}_\br$, we have 
\begin{align*}
\E_S[\hat{\Sigma}_{\br}] &= \mt{Z}_\br\E_S[\frac{1}{m}\sum_{l\in[m]}\vc{\bar{x}}_l \vc{\bar{x}}^T_l + \mt{\hat{V}}_{\br} -\mt{\hat{V}}_{\brtrain}]\mt{Z}_\br\\
&= \mt{Z}_\br(\Sigma_{\brtrain} + \mt{V}_\br - \mt{V}_\brtrain)\mt{Z}_\br = \Sigma_\br,
\end{align*}
Where the last inequality follows from \lemref{lem:sigma}.

In addition, $\E[\hat{\vc{\cov}}] = \vc{\cov}$, and $\E[\frac{1}{m}\vc{y}^T \vc{y}] = \E[Y^2]$.
Therefore 
\[
\E_S[\loss_\br(\vc{w},S)] = \vc{w}^T \Sigma_{\br} \vc{w} - 2 \vc{w}^T \mt{Z}_\br\vc{\cov} + \E[Y^2].
\]
On the other hand, we have
\begin{align*}
&\loss(\vc{w}, D_{\br}) = \E_\br[(\dotprod{\vc{w},\bar{\vc{X}}} - Y)^2]\\
&\quad= \E_{\br}[(\vc{w}^T \bar{\vc{X}} - Y)(\bar{\vc{X}}^T\vc{w} - Y)] \\
&\quad= \vc{w}^T\E[\bar{\vc{X}} \bar{\vc{X}}^T]\vc{w} - 2\vc{w}^T \E_{\br}[\bar{\vc{X}} Y] + \E[Y^2]\\
&\quad= \vc{w} \Sigma_{\br} \vc{w} - 2 \vc{w}^T \mt{Z}_\br\vc{\cov} + \E[Y^2] = \E_S[\loss_\br(\vc{w},S)].
\end{align*}
This completes the proof.
\end{proof}

The next step is to bound the rate of convergence of $\min_{\vc{w} \in \reals^d} \loss_\br(\vc{w},S)$ to $\loss(\br) = \min_{\vc{w} \in \reals^d}\loss(\vc{w},D_\br)$. We show that as $m$ grows, 
the difference between the two quantities approaches zero. The following lemma provides guarantees under the assumption that the two minimizers have a bounded norm. We will then go on to show that such a bound on the norm holds with high probability, where the bound depends on the external covariance matrix $\Sigma$.
\begin{lemma}\label{lem:rad}
Let $\alpha > 0$, and let $W_\alpha = \{ \vc{w} \in \reals^d \mid \norm{\vc{w}} \leq \alpha\}$. Let $\delta \in (0,1)$. Fix some $\vc{w}^* \in W_\alpha$, and let $\hat{\vc{w}} \in \argmin_{\vc{w} \in \reals^d} \loss_\br(\vc{w},S)$ such that $\norm{\hat{\vc{w}}}$ is minimal.
With probability at least $1-\delta$ over the draw of $S$, if $\norm{\vc{\hat{w}}} \leq \alpha$, then
\[
|\loss(\vc{w}^*,D_\br) - \loss_\br(\vc{\hat{w}},S)| \leq  O\left(\frac{\alpha^2 n_\br\ln(e/\delta)}{\sqrt{m}}\right).
\]
\end{lemma}
\begin{proof}
By \lemref{lem:rdiff}, $\loss(\vc{w},D_\br) = \E_\brtrain[\loss_\br(\vc{w},\vc{X}, Y)]$. By Rademacher complexity bounds \cite{BartlettMe02}, with probability $1-\delta$, for all $\vc{w} \in W_\alpha$,
\begin{align}
&\E_\br[\loss_\br(\vc{w},\vc{X}, Y)] \leq \label{eq:rad}\\
&\quad\loss_\br(\vc{w},S) + \cR_m(\loss_\br \circ W_\alpha, D_\brtrain)+ O\left(\beta\sqrt{\ln(1/\delta)/m}\right),\notag
\end{align}
where $\cR_m(\loss_\br \circ W_\alpha, D_\brtrain)$ is the expected Rademacher complexity of the function class $W_\alpha$ under $\loss_\br$,
and $\beta$ is the maximal value of $\loss_\br$ on the possible inputs. Under our assumptions, $\beta \leq \alpha^2 n_\br$.

We wish to bound the Rademacher complexity for our function class, defined as
\begin{equation*}%\label{eq:radcompd}
\cR_m(\loss_\br \circ W_\alpha, D_\brtrain) = \frac{1}{m}\E_{S}[\E_\sigma[|\sup_{\vc{w} \in W_\alpha} \sum_{l\in[m]}\sigma_l \loss_\br(\vc{w},\vc{x}_l, y_l)|]],
\end{equation*}
where $\sigma = (\sigma_1,\ldots,\sigma_m)$ are $m$ independent uniform $\{\pm1\}$-valued variables, and $S = ((\vc{x}_1,y_1),\ldots,(\vc{x}_m,y_m))$ is a random sample drawn i.i.d.\ from $D_\brtrain$.
Denote the components of $\loss_\br$ by 
\begin{align*}
\loss^a_\br(\vc{w},\vc{x}, y) &= (\dotprod{\mt{Z}_\br\vc{w},\vc{\bar{x}}} - y)^2,\\
\loss^{b}_\br(\vc{w},\vc{x}, y) &= \sum_{i:r[i] > 0}w[i]^2 \hat{v}(\vc{x})[i](\frac{1}{r[i]} - \frac{1}{\rtrain[i]}).
\end{align*}
We have 
\[
\cR_m(\loss_\br \circ W_\alpha, D_\brtrain) \leq \cR_m(\loss^a_\br \circ W_\alpha, D_\brtrain) + \cR_m(\loss^{b}_\br \circ W_\alpha, D_\brtrain).
\]
We bound each of these Rademacher complexities individually. 
The first term is the Rademacher complexity of the squared loss over the distribution generated by averaging over judgment vectors drawn from $D_\brtrain$. 
Standard application of the Lipschitz properties of the squared loss following \citet{BartlettMe02} provides the following bound:
\[
\cR_m(\loss^a_\br \circ W_\alpha, D_\brtrain) \leq O\left(\frac{\alpha^2 n_\br}{\sqrt{m}}\right).%d = norm{\bar{x}}^2
\]
For the second term, denote $\vc{u}_w = \mt{Z}_\br\cdot (w[1]^2,\ldots,w[d]^2)^T$,
and $\vc{\hat{v}}_\vc{x} = \mt{Z}_\br\cdot(\hat{v}(\vc{x})_1,\ldots,\hat{v}(\vc{x})_d)^T$. Then 
\[
\loss^b_\br(\vc{w},\vc{x}, y) = (\frac{1}{r[i]} - \frac{1}{\rtrain[i]})\dotprod{\vc{u}_w,\vc{\hat{v}}_\vc{x}}.
\]
Since all feature judgments are in $[-1,1]$, $\vc{\hat{v}}_\vc{x}[i] \leq 4$, thus $\norm{\vc{\hat{v}}_\vc{x}} \leq 4\sqrt{n_\br}$. In addition, 
$\norm{\vc{u}_\vc{w}} \leq \norm{\vc{w}}^2$. Lastly, $|\frac{1}{r[i]} - \frac{1}{\rtrain[i]}| \leq 1$. 
Thus, by standard Rademacher complexity bounds for the linear loss, 
\begin{align*}
\cR_m(\loss^{b}_\br \circ W_\alpha, D_\brtrain) &\leq O\left(\frac{\sup_{\vc{w} \in W_\alpha}\norm{\vc{u}_\vc{w}}\cdot\sup_{\vc{x}}\norm{\vc{\hat{v}}_\vc{x}}}{\sqrt{m}}\right) \\
&\leq O\left(\frac{\alpha^2\sqrt{n_\br}}{\sqrt{m}}\right).
\end{align*}
Summing the two terms, we have shown that 
\[
\cR_m(\loss_\br \circ W_\alpha, D_\brtrain) \leq O\left(\frac{\alpha^2n_\br}{\sqrt{m}}\right).
\]
Combining this with \eqref{eq:rad} and taking the minimum on both sides, we get 
\[
\loss(\vc{w}^*,D_\br) \leq \loss_\br(\vc{\hat{w}},S) + O\left(\frac{\alpha^2 n_\br\ln(e/\delta)}{\sqrt{m}}\right).
\]
This completes one side of the bound.
To bound $\loss_\br(\vc{\hat{w}},S) - \loss(\vc{w}^*,D_\br)$,
we note that, by the minimality of $\vc{\hat{w}}$ and by Hoeffding's inequality, with probability $1-\delta$ 
\[
\loss_\br(\vc{\hat{w}},S) \leq \loss_\br(\vc{w}^*,S) \leq \E[\loss_\br(\vc{w}^*,S)] + O(\beta\sqrt{\frac{\ln(e/\delta)}{m}}).
\]
Combining the two bounds and applying the union bound we get the statement of the lemma.
\end{proof}

We now turn to bound the norm of any $\vc{w}$ considered by our algorithm.
Let 
\[
\hat{\Sigma} = \frac{1}{m}\sum_{l\in[m]}\vc{\bar{x}}_l \vc{\bar{x}}^T_l - \mt{\hat{V}}_\brtrain.
\]
First, we relate the smallest eigenvalue of $\hat{\Sigma}$ to that of the true $\Sigma$.
\begin{lemma}\label{lem:siginflam}
Let $\delta \in (0,1)$. With probability at least $1-\delta$, 
\begin{align}
&\lambda_{\min}(\submat_\br(\hat{\Sigma}))\geq \label{eq:siginflam}\\
&\quad\lambda_\br - \sqrt{O(\ln(1/\delta) + n_\br\ln(n_\br m))/m}.\notag
\end{align}
\end{lemma}
\begin{proof}
It can be shown \citep[see e.g.][]{Sabato12}, by using an $\epsilon$-net of vectors on the unit sphere, that for a symmetric random matrix $\mt{S} \in \reals^{n\times n}$, and positive $\beta,\epsilon,\gamma$,
\begin{align}\label{eq:enet}
&\P[\lambda_{\min}(\mt{S}) \leq \beta - \epsilon \gamma] \leq \\
&\quad\P[\lambda_{\max}(\mt{S}) > \gamma] + O((n/\epsilon)^n)\min_{\vc{u}:\norm{\vc{u}} = 1}\P[\vc{u}^T \mt{S} \vc{u} \leq \beta].\notag
\end{align}
In our case, we have $\mt{S} = \submat_\br(\hat{\Sigma})$ and $n = n_\br$. Due to the boundedness of the judgments, all the entries of $\mt{S}$ are at most $1$.
Therefore $\lambda_{\max}(\mt{S}) \leq n_\br$. We thus let $\gamma = n_\br$, so that $\P[\lambda_{\max}(\mt{S}) > \gamma]  =0$.
To bound \mbox{$\P[\vc{u}^T \mt{S} \vc{u} \leq \beta]$} for $\vc{u}$ such that $\norm{\vc{u}} = 1$, note that 
\begin{align*}
&\vc{u}^T \mt{S} \vc{u} = \frac{1}{m}\Big(\sum_{l \in [m]}\dotprod{\vc{u}, \mt{Z}_\br\bar{\vc{x}}_l}^2 -\sum_{i: r[i] > 0} u[i]^2 v(\vc{x}_l)[i]/\rtrain[i]\Big).
\end{align*}
By McDiarmid's inequality \cite{McDiarmid89}, for any $t > 0$,
\[
\P[\vc{u}^T \mt{S} \vc{u} \leq \E[\vc{u}^T \mt{S} \vc{u}] - t] \leq \exp(-2t^2/m\Delta^2),
\]
Where $\Delta$ is the maximal difference between $\vc{u}^T \mt{S} \vc{u}$ with some $\vc{x}_1,\ldots,\vc{x}_m$ and $\vc{u}^T \mt{S} \vc{u}$ with $\vc{x}[i]$ replaced by some $\vc{x}'[i]$. It is easy to verify that due to the boundedness of all $\vc{x}_l$, 
$\Delta \leq O(n_\br/m)$. Further, $\vc{u}^T \E[\mt{S}] \vc{u} = \vc{u}^T \submat_\br(\Sigma) \vc{u} \geq \lambda_{\min}(\submat_\br(\Sigma)) = \lambda_\br$.
Thus, 
\[
\P[\vc{u}^T \mt{S} \vc{u} \leq \lambda_\br - t] \leq \exp(-2mt^2/n_\br^2).
\] 
Substituting into \eqref{eq:enet}, we get
\begin{align*}
&\P[\lambda_{\min}(\mt{S}) \leq \lambda_\br - t - \epsilon n_\br]  \\
&\quad \leq O((n_\br/\epsilon)^{n_\br})\exp(-2mt^2/n_\br^2)\\
&\quad= \exp(-2mt^2 + n_\br\ln(n_\br/\epsilon)).
\end{align*}
Letting $\epsilon = t/n_\br$ and solving for $\delta$ we get the statement of the theorem.
\end{proof}

Using \lemref{lem:siginflam}, we can now bound the norms of a minimizer of $\loss(\vc{w},D_\br)$ and of a minimizer of $\loss_\br(\vc{w},S)$ with high probability.
\begin{lemma}\label{lem:normw}
Let $\vc{w} \in \reals^d$ be a minimum-norm minimizer for $\loss_\br(\vc{w}, S)$, that is, let $M$ be the set of minimizers for $\loss_\br(\vc{w},S)$, and let $\hat{\vc{w}} \in \argmin_{\vc{w} \in M}\norm{\vc{w}}$. Let $\vc{w}^*$ be a minimum-norm minimizer for $\loss_\br(\vc{w},S)$.
If $m \geq \tilde{\Omega}((\ln(1/\delta) + n_\br\ln(n_\br))/\lambda_\br^2)$, then with probability at least $1-\delta$,  $\submat_\br(\hat{\Sigma})$ is positive definite, and
\begin{align*}
&\norm{\hat{\vc{w}}} \leq\sqrt{n_\br}\left(\lambda_\br - \sqrt{\frac{O(\ln(1/\delta) + n_\br\ln(n_\br m))}{m}}\right)^{-1},\text{ and }\\
&\norm{\vc{w}^*} \leq \sqrt{n_\br}/\lambda_\br.
\end{align*}
\end{lemma}
Further, this holds simultaneously for any vector $\br' \in \nats^d$ with the same support as $\br$. 
\begin{proof}
Any minimum-norm minimizer for $\loss_\br(\vc{w}, S)$ has $w[i] = 0$ whenever $r[i] = 0$. Thus, we may assume w.l.o.g. that $r[i] > 0$ for all $i$ by deleting the coordinates with $r[i] = 0$.
We thus have 
\[
\loss_\br(\vc{w}, S) = \vc{w}^T\hat{\Sigma}_\br\vc{w} - 2\vc{w}^T\hat{\vc{\cov}} + \frac{1}{m}\vc{y}^T \vc{y}.
\]
If $\hat{\Sigma}_\br$ is positive definite, then 
the minimizer of $\loss_\br(\vc{w}, S)$ is $\vc{\hat{w}} = \hat{\Sigma}_\br^{-1}\vc{\hat{\cov}}$.
Thus 
\begin{align}
\norm{\vc{\hat{w}}} &\leq \norm{\hat{\cov}}\lambda_{\max}(\hat{\Sigma}_\br^{-1})\label{eq:normw1} \\
&= \norm{\hat{\cov}}/\lambda_{\min}(\hat{\Sigma}_\br) \leq \sqrt{n_\br}/\lambda_\br.\notag
\end{align}
Now, $\hat{\Sigma}_\br = \hat{\Sigma} + \hat{\mt{V}}_\br$. Since $\hat{\mt{V}}_\br \succeq 0$, we have
$\lambda_{\min}(\hat{\Sigma}_\br) \geq \lambda_{\min}(\hat{\Sigma})$. $\lambda_{\min}(\hat{\Sigma})$ can be bounded from below by \lemref{lem:siginflam}.
Substituting \eqref{eq:siginflam} in \eqref{eq:normw1} we get the first desired inequality.
For the second inequality, it suffices to note that if $\submat_\br(\Sigma)$ is not singular, then $\vc{w}^* = \submat_\br(\Sigma)^{-1} \vc{\cov}$, 
thus $\norm{\vc{w}^*} \leq \lambda^{-1}_{\min}(\submat_\br(\Sigma))\norm{\cov} \leq \sqrt{n_\br}/\lambda_\br$.
\end{proof}

Combining \lemref{lem:rad} and \lemref{lem:normw} and applying the union bound we immediately get the following theorem.
\begin{theorem}\label{thm:singleconv}
Let $S$ be a training sample of size $m$ drawn from $D_\brtrain$, where $\rtrain[i] \geq 2$ for all $i$ in $[d]$. 
Fix $\br \in \nats^d$.
Let $\hat{\vc{w}}$ be a minimum-norm minimizer of $\loss_\br(\vc{w},S)$.
Let $\delta \in (0,1)$. If $m \geq \tilde{\Omega}((\ln(1/\delta) + n_\br\ln(n_\br))/\lambda_\br^2)$, 
then with probability at least $1-\delta$, $\submat_\br(\hat{\Sigma})$ is positive definite and
\begin{equation*}
|\loss(\br) - \loss_\br(\vc{\hat{w}},S)| \leq O\left(\frac{n_\br^2\lambda^{-2}_\br\ln(e/\delta)}{\sqrt{m}}\right).
\end{equation*}
Moreover, the positive-definiteness holds simultaneously for all $\br'$ with the same support as $\br$.
%todo can this be made to hold for all r of same support? key is in rademacher part
\end{theorem}

Finally, we prove our main result, \thmref{thm:convergence}.
\begin{proof}[Proof of \thmref{thm:convergence}]
We start by proving the result for the full algorithm. 
Assume that $m \geq \tilde{\Omega}((\ln(1/\delta) + n_\br\ln(n_\br))/\lambda^2)$, and consider a fixed $\br \in R_B$. Recall that 
\[
\loss_\br(\vc{w},S) = \vc{w}^T\hat{\Sigma}_\br\vc{w} - 2\vc{w}^T\mt{Z}_\br\hat{\vc{\cov}} + \frac{1}{m}\vc{y}^T \vc{y}.
\]
The full algorithm ignores, for each examined repeat vector, all the matrix and vector entries that correspond to features with zero repeats. Thus we may assume w.l.o.g. that for all $i$, $r[i] > 0$. \thmref{thm:singleconv} guarantees with probability $1-\delta$, that $\hat{\Sigma}$ is positive definite with probability $1-\delta$. It follows that $\hat{\Sigma}^p = \hat{\Sigma}$, therefore $\hat{\Sigma}^p_\br = \hat{\Sigma}_\br$.
We thus have 
\[
\loss_\br(\vc{w},S) = \vc{w}^T\hat{\Sigma}^p_\br\vc{w} - 2\vc{w}^T\hat{\vc{\cov}} + \frac{1}{m}\vc{y}^T \vc{y}.
\]
The minimizer of $\loss_\br(\vc{w},S)$ is $\vc{\hat{w}} = (\hat{\Sigma}^p_\br)^{-1}\hat{\vc{\cov}}$.
Substituting this solution in $\loss_\br(\vc{w},S)$, we get $\min_{\vc{w}\in \reals^d}\loss_\br(\vc{w},S) = -\hat{\vc{\cov}}^T(\hat{\Sigma}^p_\br)^{-1}\hat{\vc{\cov}} + \frac{1}{m}\vc{y}^T \vc{y} = -\hat{\obj}_f(\br) + \frac{1}{m}\vc{y}^T \vc{y}$.
If the features are uncorrelated then 
\begin{align*}
\min_{\vc{w}\in \reals^d}\loss_\br(\vc{w},S) &= -\hat{\vc{\cov}}^T(\hat{\Sigma}^p_\br)^{-1}\hat{\vc{\cov}} + \frac{1}{m}\vc{y}^T \vc{y} \\
&= -\hat{\obj}_f(\br) + \frac{1}{m}\vc{y}^T \vc{y}.
\end{align*}
Since all labels $Y$ are bounded in $[-1,1]$, with probability $1-\delta$,  $|\frac{1}{m}\vc{y}^T \vc{y} - \E[Y^2]| \leq O(\sqrt{\ln(1/\delta)/m})$. 
Thus it suffices to bound $|\loss(\br) - \min_{\vc{w} \in \reals^d}\loss_\br(\vc{w},S)|$. This bound is given by \thmref{thm:singleconv}. We now apply this bound simultaneously to all the possible $\br \in R_B$. There are
\[
\sum_{i\in [B]} \binom{d+i-1}{d-1} = \sum_{i\in [B]} \binom{d+i-1}{i-1} \leq B(d+B)^{\bar{B}}
\]
such combinations, thus the bound in \thmref{thm:singleconv} holds simultaneously for all $\br \in R_B$, by dividing $\delta$ with this upper bound. For the requirement on the size of $m$ we only need a union bound on the number of possible supports for $\br$, which is bounded by $d^{\bar{B}}$. Finally, by noting that for all $\br \in R_B$, $n_\br \leq \bar{B}$, we get the desired uniform convergence bound.

Now, consider the Scoring algorithm. If $\Sigma$ is diagonal, then $\loss(\vc{w},D_\br)$ decomposes into a sum of $n_\br$ independent losses over single-dimensional predictors with a single-dimensional covariance `matrix' equal to the scalar $\sigma^2[i]$ for feature $i$. The loss minimized by the Scoring algorithm decomposes similarly, using the covariance `matrix' $\hat{\sigma}^2[i]$. Thus, we apply the convergence bound of \thmref{thm:singleconv} to show convergence of each of the components individually, with $n_\br = 1$ and a union bound over $B$ possible values of $r[i]$, and then apply a union bound over $d$ components to get simultaneous convergence of the parts of the loss. For the requirement on the size of $m$ we only need a union bound over the number of components $d$.
\end{proof}

\subsection{Proof of \thmref{thm:greedy}}\label{sec:scoringanalysis}
First, we prove the more general \lemref{lem:generalgreedy}.
We actually prove an equivalent mirror image of \lemref{lem:generalgreedy}, by assuming the $g_i$ are convex non-increasing and proving that a greedy algorithm minimizes $f(\br)$.
We formally define the greedy algorithm as follows: 
\begin{enumerate}
\item $\br_0 \leftarrow (0,\ldots,0)$.
\item For $t = 1$ to $B$, let $\br_t = \br_{t-1} + \vc{e}_{i}$, where $i$ is the coordinate of $\br$ that decreases $f(\br)$ the most.
\item Return $\br_B$.
\end{enumerate}

\begin{proof}[Proof of \lemref{lem:generalgreedy}]
Let $\br^* \in \argmin_{\br \in R_B} f(\br)$. 
Since $g_i(\cdot)$ are all non-increasing, we may choose $\br^*$ such that $\sum_{i\in [d]}r^*[i] = B$. Let $\br$ be a solution returned by the greedy algorithm listed in the theorem statement. 
Consider the iterations $t_1 < \ldots < t_n \in [B]$ such that the index $i_k$ selected by the algorithm at iteration $t_k$ satisfies $r_{t_k}[i_k] > r^*[i_k]$, so that it causes $\br$ to increase this coordinate more than its value in $\br^*$. Let $j_1,\ldots,j_n \in [d]$ be a series of alternative indices such that 
\[
\br = \br^* + \sum_{k \in [n]} \vc{e}_{i_k} - \sum_{k \in [n]} \vc{e}_{j_k}.
\]
Denote 
\[
\br^*_L = \br^* + \sum_{k \in [L]} \vc{e}_{i_k} - \sum_{k \in [L]} \vc{e}_{j_k}.
\]
Note that for all $L < n$, $r[j_L] < r_{L-1}^*[j_L]$.

We prove by induction that for all $L \in [n]$,
\[
f(\br^*_L) = f(\br^*).
\]
When setting $L = n$ we will get $f(\br) = f(\br^*)$.

The claim trivially holds $L = 0$. Now assume it holds for $L-1$, and consider $L$.
Denote for brevity $t = t_{L-1}$, $i = i_L$ and $j = j_L$.
Since the algorithm selected $i$ over $j$ at iteration $t+1$, we have that
\[
f(\br_t+\vc{e}_i) \leq f(\br_t + \vc{e}_j).
\]
Subtracting $\sum_{l \in [d]}g_l(r_{t}[l])$ from both sides we get 
\[
g_i(r_t[i]+1) - g_i(r_t[i]) \leq g_j(r_t[j]+1) - g_j(r_t[j]).
\]
It follows that 
\[
0 \leq g_i(r_t[i]) - g_i(r_t[i] + 1) + g_j(r_t[j]+1) - g_j(r_t[j]).
\]
Since $r_t[i] \geq r_{L-1}^*[i]$, by the convexity of $g_i$,
\[
g_i(r_t[i]) - g_i(r_t[i] + 1)\leq g_i(r_{L-1}^*[i]) - g_i(r_{L-1}^*[i] + 1).
\]
In addition, $r_t[j] +1 \leq r_{L-1}^*[j]$, therefore, again by convexity,
\[
g_j(r_t[j] + 1) - g_j(r_t[j]) \leq g_j(r_{L-1}^*[j]) - g_j(r_{L-1}^*[j] - 1).
\]
It follows that 
\begin{align*}
0 \leq &g_i(r_{L-1}^*[i]) - g_i(r_{L-1}^*[i] + 1) + \\
&g_j(r_{L-1}^*[j]) - g_j(r_{L-1}^*[j] - 1).
\end{align*}
Therefore
\[
0 \leq f(\br^*_{L-1}) - f(\br^*_{L-1} + \vc{e}_i - \vc{e}_j).
\]
By the induction hypothesis $f(\br^*) = f(\br^*_{L-1})$. In addition, $\br^*_{L-1} + \vc{e}_i - \vc{e}_j = \br^*_L$. it follows that 
\[
f(\br^*_L)\leq f(\br^*). 
\]
Since $\br^*$ is optimal for $f$, it follows that this inequality must hold with equality,
thus proving the induction hypothesis.
\end{proof}

\begin{proof}[Proof of \thmref{thm:greedy}]
We have 
\[
\hat{\obj}(\br) = \sum_{i: r[i] > 0} \frac{\hat{\cov}[i]^2}{\hat{\sigma}^2[i] + \hat{v}_i/r[i]}. 
\]
Define $f_i:\nats_+ \rightarrow \reals$ by
\[
f_i(x) = \frac{x \cdot \hat{\cov}[i]^2}{x\cdot \hat{\sigma}^2[i] + \hat{v}[i]}.
\]
Then 
\[
\hat{\obj}(\br) = \sum_{i: r[i] > 0} f_i(x).
\]
We will define $g_i(x):\reals \rightarrow \reals$ so that each $g_i$ is concave and
\begin{equation}\label{eq:g}
\hat{\obj}(\br) = \sum_{i\in[d]} g_i(r[i]).
\end{equation}
We consider two cases: (1) If $\hat{v}[i] > 0$, then $g_i(x)$ is the natural extension of $f_i(x)$ to the reals.
(2) If $\hat{v}[i] = 0$, $f_i$ is a positive constant for all positive integers. Let $g_i(x) = f_i(1)$ for all $x \geq 1$,
and let $g_i(x) = x/f_i(1)$.  

In both cases, $g_i(x)$ is concave with $g_i(x) = f_i(x)$ for $x \in \nats_+$ and $g_i(0) = 0$. Therefore \eqref{eq:g} holds.
In addition, in both cases $g_i(x)$ is monotonic non-decreasing. Therefore by \lemref{lem:generalgreedy}, the greedy algorithm maximizes $\hat{\obj}(\br)$ subject to $\br \in R_B$.
\end{proof}

\end{document}